\documentclass{article}

\usepackage{a4wide}
\usepackage{authblk}
\usepackage{amsthm}

\newtheorem{thm}{Theorem}[section]
\newtheorem{theorem}[thm]{Theorem}
\newtheorem{lemma}[thm]{Lemma}
\newtheorem{corollary}[thm]{Corollary}

\theoremstyle{definition}

\usepackage[utf8]{inputenc} % allow utf-8 input
\usepackage[T1]{fontenc}  % use 8-bit T1 fonts
\usepackage{url}      % simple URL typesetting
\usepackage{booktabs}    % professional-quality tables
\usepackage{amsfonts}    % blackboard math symbols
\usepackage{nicefrac}    % compact symbols for 1/2, etc.
\usepackage{microtype}   % microtypography
\usepackage{xspace}

\usepackage{comment}
\usepackage{array}
\usepackage{units}
\usepackage{bm}
\usepackage{amsmath}
\usepackage{amssymb}
\usepackage{graphicx}
\usepackage[section]{algorithm}
\usepackage{algorithmicx}
\usepackage{algpseudocode}

\usepackage[table]{xcolor}

\renewcommand{\vec}[1]{\bm {#1}}

\let\originalleft\left
\let\originalright\right
\renewcommand{\left}{\mathopen{}\mathclose\bgroup\originalleft}
\renewcommand{\right}{\aftergroup\egroup\originalright}

\usepackage{pifont}% http://ctan.org/pkg/pifont
\newcommand{\cmark}{\ding{51}}%
\newcommand{\xmark}{\ding{55}}

\usepackage{color}

\usepackage{todonotes}
\usepackage[normalem]{ulem}
\usepackage{etoolbox}
\newtoggle{FinalVersion}
\togglefalse{FinalVersion}

\iftoggle{FinalVersion}{
 
 \newcommand{\deleted}[1]{{}}
}{
 
 \newcommand{\deleted}[1]{{\sout{\color{red}#1}}} 
}

\newcommand{\atoppush}{TopPush\xspace}
\newcommand{\atoppushk}{TopPushK\xspace}
\newcommand{\agrill}{Grill\xspace}
\newcommand{\apatmat}{Pat\&Mat\xspace}
\newcommand{\atopmeank}{TopMean\xspace}
\newcommand{\anpA}{Grill-NP\xspace}
\newcommand{\anpB}{Pat\&Mat-NP\xspace}
\newcommand{\anpC}{TopMean-NP\xspace}

\newcommand{\toppush}{\emph{TopPush}\xspace}
\newcommand{\toppushk}{\emph{TopPushK}\xspace}
\newcommand{\grill}{\emph{Grill}\xspace}
\newcommand{\patmat}{\emph{Pat}\&\emph{Mat}\xspace}
\newcommand{\topmeank}{\emph{TopMean}\xspace}
\newcommand{\npA}{\emph{Grill-NP}\xspace}
\newcommand{\npB}{{\emph{Pat}\&\emph{Mat-NP}}\xspace}
\newcommand{\npC}{{\emph{TopMean-NP}}\xspace}

\newcommand{\norm}[1]{\|#1\|}

\newcommand{\tp}{\textnormal{tp}}
\newcommand{\tn}{\textnormal{tn}}
\newcommand{\fp}{\textnormal{fp}}
\newcommand{\fn}{\textnormal{fn}}
\newcommand{\tps}{\overline{\textnormal{tp}}}
\newcommand{\tns}{\overline{\textnormal{tn}}}
\newcommand{\fps}{\overline{\textnormal{fp}}}
\newcommand{\fns}{\overline{\textnormal{fn}}}

\newcommand{\Xcal}{\mathcal{X}}

\usepackage{colortbl}
\colorlet{tableheadcolor}{gray!25} % Table header colour = 25% gray
\newcommand{\headcol}{\rowcolor{tableheadcolor}}

\usepackage{pgfplots}
\usepackage{pgfplotstable}
\usetikzlibrary{pgfplots.groupplots}
\pgfplotsset{compat=1.3}
\pgfplotsset{
    line1/.style={%
        smooth, red, thick},
    line2/.style={%
        smooth, blue, thick, dashed},
    line3/.style={%
        black, no marks, dotted},
    line4/.style={%
        black, no marks, dashdotted},
    col1/.style={%
        xmin=0, xmax=1,
        xtick={0,0.2,0.4,0.6,0.8}},
    col2/.style={%
        xmin=0, xmax=0.1,
        xtick={0,0.02,0.04,0.06,0.08},
        xticklabels={0,0.02,0.04,0.06,0.08}},
    col3/.style={%
        xmin=0, xmax=0.01,
        xtick={0,0.002,0.004,0.006,0.008},
        xticklabels={0,0.002,0.004,0.006,0.008},
        scaled x ticks=false},
    legendStyleA/.style={%
        column sep = 10pt,
        legend columns = -1,
        legend to name = grouplegendA},
}
\pgfplotstableread[col sep = comma]{Ptau_Gisette.csv}\tabBA
\pgfplotstableread[col sep = comma]{Ptau_Hepmass.csv}\tabCA
\pgfplotstableread[col sep = comma]{Ptau_NetFlow.csv}\tabAA
\pgfplotstableread[col sep = comma]{Ptau_NetFlow_Small.csv}\tabAB
\pgfplotstableread[col sep = comma]{Ptau_Gisette_Small.csv}\tabBB
\pgfplotstableread[col sep = comma]{Ptau_Hepmass_Small.csv}\tabCB

\newcommand{\theoremNN}[2]{\noindent\textbf{Theorem \ref{#1} (page \pageref{#1})}\ \emph{#2} \\}

\begin{document}

\title{General Framework for Binary Classification on Top Samples}

\author[1]{Luk\'a\v{s} Adam\thanks{3@ieee.org}}
\author[2]{V\'aclav M\'acha}
\author[2]{V\'aclav \v{S}m\'idl}
\author[3]{Tom\'a\v{s} Pevn\'y}

\affil[1]{Department of Computer Science and Engineering, Southern University of Science and Technology, Shenzhen, China}
\affil[2]{Institute of Information Theory and Automation, Czech Academy of Sciences, Prague, Czech Republic}
\affil[3]{Department of Computer Science, Czech Technical University in Prague, Prague, Czech Republic}
\renewcommand\Authands{ and }

\maketitle

\begin{abstract}%   <- trailing '%' for backward compatibility of .sty file
Many binary classification problems minimize misclassification above (or below) a threshold. We show that instances of ranking problems, accuracy at the top or hypothesis testing may be written in this form. We propose a general framework to handle these classes of problems and show which known methods (both known and newly proposed) fall into this framework. We provide a theoretical analysis of this framework and mention selected possible pitfalls the methods may encounter. We suggest several numerical improvements including the implicit derivative and stochastic gradient descent. We provide an extensive numerical study. Based both on the theoretical properties and numerical experiments, we conclude the paper by suggesting which method should be used in which situation.
\end{abstract}



\section{Introduction}\label{sec:Motivation}

Many binary classification problems focus on separating the dataset by a linear hyperplane $\bm w^\top \bm x-t$. A sample $\bm x$ is deemed to be positive or relevant (depending on the application) if its score $\bm w^\top \bm x$ is above a threshold $t$. Multiple problem categories belong to this framework:
\begin{itemize}\itemsep 0pt
%\item \textit{Support vector machines} where the threshold equals the offset of the separating hyperplane. The accuracy is evaluated for the whole dataset.
\item \textit{Ranking problems} select the most relevant samples and rank them. To each sample, a numerical score is assigned and the ranking is performed based on this score. Often, only scores above a threshold are considered.
\item \textit{Accuracy at the Top} is similar to ranking problems. However, instead of ranking the most relevant samples, it only maximizes the accuracy (equivalently minimizes the misclassification) in these top samples. The prime examples of both categories include search engines or problems where identified samples undergo expensive post-processing such as human evaluation.
\item \textit{Hypothesis testing} states a null and an alternative hypothesis. The Neyman-Pearson problem minimizes the Type II error (the null hypothesis is false but it fails to be rejected) while keeping the Type I error (the null hypothesis is true but is rejected) small. If the null hypothesis states that a sample has the positive label, then Type II error happens when a positive sample is below the threshold and thus minimizing the Type II error amounts to minimizing the positives below the threshold.
\end{itemize}
All these three applications may be written (possibly after a reformulation) in a similar form as a minimization of the false-negatives (misclassified positives) above a threshold. They only differ in the way they define the threshold. Despite this striking similarity, they are usually considered separately in the literature. The main goal of this paper is to provide a unified framework for these three applications and perform its theoretical and numerical analysis.

The goal of the ranking problems is to rank the relevant samples higher than the non-relevant ones. A prototypical example is the RankBoost \cite{freund2003efficient} maximizing the area under the ROC curve, the Infinite Push \cite{agarwal2011infinite} or the $p$-norm push \cite{rudin.2009} which concentrate on the high-ranked negatives and push them down. Since all these papers include pairwise comparisons of all samples, they can be used only for small datasets. This was alleviated in \cite{Li_TopPush}, where the authors performed the limit $p\to\infty$ in $p$-norm push and obtained the linear complexity in the number of samples. Moreover, since the $l_{\infty}$-norm is equal to the maximum, this method falls into our framework with the threshold equal to the largest score computed from negative samples.

Accuracy at the Top ($\tau$-quantile) was formally defined in \cite{boyd2012accuracy} and maximizes the number of relevant samples in the top $\tau$-fraction of ranked samples. When the threshold equals the top $\tau$-quantile of all scores, this problem falls into our framework. The early approaches aim at solving approximations, for example, \cite{Joachims:2005:SVM:1102351.1102399} optimizes a convex upper bound on the number of errors among the top samples. Due to the presence of exponentially many constraints, the method is computationally expensive. \cite{boyd2012accuracy} presented an SVM-like formulation which fixes the index of the quantile and solves $n$ problems. While this removes the necessity to handle the (difficult) quantile constraint, the algorithm is computationally infeasible for a large number of samples. \cite{kar2015surrogate} derived upper approximations, their error bounds and solved these approximations. \cite{Grill_2016} proposed the projected gradient descent method where after each gradient step, the quantile is recomputed. \cite{Eban_2017} suggested new methods for various criteria and argued that they keep desired properties such as convexity. \cite{mackey2018constrained} generalized this paper by considering a more general setting. \cite{tasche2018plug} showed that accuracy at the top is maximized by thresholding the posterior probability of the relevant class.

Closest approach to our framework is \cite{lapin.2015,lapin2018analysis}, where the authors considered multi-class classification problems and their goal was to optimize the performance on the top few classes. Thus, the top was considered for classes and not for samples as in our case.

The paper is organized as follows. In Section \ref{sec:framework} we introduce the unified framework and then show how the three applications above fall into it. To each application, we state several numerical methods, some of them are known, some are modifications of known methods and some are new.

Section \ref{sec:example} presents a simple example which shows the differences between these methods and highlights the major problem that some of the methods may have the global minimum at $\bm w=0$. Since the weights form the normal of the separating hyperplane, this solution does not provide any information.

In Section \ref{sec:theory} we perform a theoretical analysis of the unified framework. First, we focus on convexity which ensures that no local minima are present. Second, we build on the example from the previous section and analyze the case when zero weights are the global minimum. We show that the higher the threshold, the more vulnerable a convex method is to have the global minimum at zero. Based on this result, we derive how the methods are chained with respect to increasing thresholds.

Section \ref{sec:num1} presents numerical considerations. We briefly mention the use of the stochastic gradient descent for our problems, describe the performance criteria and hyperparameter choice. Moreover, since the threshold $t$ depends on the weights $\bm w$, we follow the implicit programming technique and remove the threshold constraint. At the same time, we show how to compute derivatives for the reduced problem. Finally, we present the datasets on which we perform the numerical experiments in Section \ref{sec:num2}. Here, we present standard precision-recall curves and then concentrate on a comparison on the presented methods.

The paper is concluded by a recommendation stating which methods should be used in which situations. This recommendation is based both on the theoretical and numerical analysis. To keep the brevity of the paper, we postpone multiple results to the Appendix. The reader is welcome to refer to our codes online.\footnote{\texttt{https://github.com/VaclavMacha/ClassificationOnTop}}

\section{Framework for Minimizing Missclassification Above a Threshold}\label{sec:framework}

Many important binary classification problems minimize the number of misclassified samples below (or above) certain threshold. Since these problems are usually considered separately, in this section, we provide a unified framework for their handling and present several classification problems falling into this framework.

For samples $\bm x$, we consider the linear classifier $f(\bm w)=\bm w^\top \bm x-t$, where $\bm w$ is the normal vector to the separating hyperplane and $t$ is a threshold. The most well-known example is the support vector machines where $t$ is a free variable. In many cases the threshold $t$ is computed from the scores $z=\bm w^\top \bm x$. For example, the \toppush method from \cite{Li_TopPush} sets the threshold $t$ to the largest score $z^-$ corresponding to negative samples and in the method from \cite{Grill_2016} the threshold equals the quantile of all scores.

To be able to determine the missclassification above and below the threshold $t$, we define the true-positive, false-negative, true-negative and false-positive counts by
\begin{equation}\label{eq:defin_counts}
\aligned
\tp(\bm w,t) & = \sum_{\bm x\in\mathcal X^+}\left[\bm w^\top \bm x-t \ge 0\right], &
\fn(\bm w,t) & = \sum_{\bm x\in\mathcal X^+}\left[\bm w^\top \bm x-t < 0\right], \\
\tn(\bm w,t) & = \sum_{\bm x\in\mathcal X^-}\left[\bm w^\top \bm x-t < 0\right], &
\fp(\bm w,t) & = \sum_{\bm x\in\mathcal X^-}\left[\bm w^\top \bm x-t \ge 0\right].
\endaligned
\end{equation}
Here $[\cdot]$ is the 0-1 loss (Iverson bracket, characteristic function) which is equal to $1$ if the argument is true and to $0$ otherwise. Moreover, $\mathcal{X}/\mathcal{X}^{+}/\mathcal{X}^{-}$ denotes the sets of all/positive/negative samples and by $n/n^{+}/n^{-}$ their respective sizes. 

Since the misclassified samples below the threshold are the false-negatives, we arrive at the following problem
\begin{equation}\label{eq:problem1}
\aligned
{\rm minimize} & \quad\frac{1}{n^{+}}\fn(\bm w,t)\\
\textnormal{subject to} & \quad \text{threshold }t\text{ is a function of }\{\bm w^\top \bm x\}.
\endaligned
\end{equation}
As the 0-1 loss in \eqref{eq:defin_counts} is discontinuous, problem \eqref{eq:problem1} is difficult to handle. The usual approach is to employ a surrogate function such as the hinge loss function defined by
\begin{equation}\label{eq:defin_surrogate}
\aligned
l_{\rm hinge}\left(z\right) & =\max\left\{ 0,1+z\right\}. \\
%l_{\log}\left(z\right) & =\frac{1}{\log2}\log\left(1+e^z\right).
\endaligned
\end{equation}
In the text below, the symbol $l$ denotes any convex non-negative non-decreasing function with $l(0)=1$. Using the surrogate function, the counts \eqref{eq:defin_counts} may be approximated by their surrogate counterparts
\begin{equation}\label{eq:defin_counts_surr}
\aligned
\tps(\bm w ,t) & = \sum_{\bm x\in\mathcal{X}^+}l(\bm w^\top \bm x-t), &
\fns(\bm w ,t) & = \sum_{\bm x\in\mathcal{X}^+}l(t-\bm w^\top \bm x), \\
\tns(\bm w ,t) & = \sum_{\bm x\in\mathcal{X}^-}l(t-\bm w^\top \bm x),&
\fps(\bm w ,t) & = \sum_{\bm x\in\mathcal{X}^-}l(\bm w^\top \bm x-t).
\endaligned
\end{equation}
Since $l(\cdot)\ge[\cdot]$, the surrogate counts \eqref{eq:defin_counts_surr} provide upper approximations of the true counts \eqref{eq:defin_counts}. Replacing the counts in \eqref{eq:problem1} by their surrogate counterparts and adding a regularization results in
\begin{equation}\label{eq:problem2}
\aligned
{\rm minimize} & \quad\frac{1}{n^{+}}\fns(\bm w,t) + \frac\lambda2\norm{\bm w}^2 \\
\textnormal{subject to} & \quad \text{threshold }t\text{ is a function of }\{\bm w^\top \bm x\}.
\endaligned
\end{equation}

In the rest of this section, we list methods which fall into the framework of \eqref{eq:problem1} and \eqref{eq:problem2}. We divide the methods into three categories based on the criterion the methods optimize. To all categories, we provide several methods which include known methods, new methods and modifications of known methods.

\subsection{Methods based on pushing positives to the top}\label{sec:obj1}

The first category of methods falling into our framework \eqref{eq:problem1} and \eqref{eq:problem2} are ranking methods which attempt to put as many positives (relevant samples) to the top as possible. Specifically, for each sample $\bm x$, they compute the score $z=\bm w^\top \bm x$ and then sort the vector $\bm z$ into $\bm z_{[\cdot]}$ with decreasing components $z_{[1]}\ge z_{[2]}\ge\dots\ge z_{[n]}$. Since the number of positives on top amounts to the number of positives above the highest negative, this may be written as
\begin{equation}\label{eq:problem_top}
\aligned
{\rm maximize} & \quad\frac{1}{n^{+}}\tp(\bm w,t) \\
\textnormal{subject to} & \quad t = z_{[1]}^-, \\
& \quad \text{components of }\bm z^-\text{ equal to }z^-=\bm w^\top \bm x^-\text{ for }\bm x^- \in \Xcal^-.
\endaligned
\end{equation}
Note that since $\tp(\bm w,t)+\fn(\bm w,t)=n^+$, maximizing the true-positives is equivalent to minimizing the false-negatives. Thus, we observe that \eqref{eq:problem_top} is equivalent to
\begin{equation}\label{eq:problem_top2}
\aligned
{\rm minimize} & \quad\frac{1}{n^{+}}\fn(\bm w,t) \\
\textnormal{subject to} & \quad t = z_{[1]}^-, \\
& \quad \text{components of }\bm z^-\text{ equal to }z^-=\bm w^\top \bm x^-\text{ for }\bm x^- \in \Xcal^-.
\endaligned
\end{equation}
As $t$ is a function of the scores $z=\bm w^\top \bm x$, problem \eqref{eq:problem_top} is a special case of \eqref{eq:problem1}.

\subsubsection{TopPush}

The \toppush method \cite{Li_TopPush} replaces the false-negatives in \eqref{eq:problem_top2} by their surrogate and adds regularization to arrive at
\begin{equation}\label{eq:problem_toppush}
\aligned
{\rm minimize} & \quad\frac{1}{n^{+}}\fns(\bm w,t) + \frac\lambda2\norm{\bm w}^2 \\
\textnormal{subject to} & \quad t = z_{[1]}^-, \\
& \quad \text{components of }\bm z^-\text{ equal to }z^-=\bm w^\top \bm x^-\text{ for }\bm x^- \in \Xcal^-.
\endaligned
\end{equation}
Note that this falls into the framework of \eqref{eq:problem2}.

\subsubsection{TopPushK}

As we will show in Section \ref{sec:example}, the \toppush method is sensitive to outliers and mislabelled data. To robustify it, we follow the idea from \cite{lapin.2015} and propose to replace the largest negative score by the mean of $k$ largest negative scores. This results in
\begin{equation}\label{eq:problem_toppushk}
\aligned
{\rm minimize} & \quad\frac{1}{n^{+}}\fns(\bm w,t) + \frac\lambda2\norm{\bm w}^2 \\
\textnormal{subject to} & \quad t = \frac 1k(z_{[1]}^- + \dots + z_{[k]}^-), \\
& \quad \text{components of }\bm z^-\text{ equal to }z^-=\bm w^\top \bm x^-\text{ for }\bm x^- \in \Xcal^-.
\endaligned
\end{equation}
We used the mean of highest $k$ negative scores instead of the value of the $k$-th negative score to preserve convexity as indicated in Section \ref{sec:convexity}.

\subsection{Accuracy at the Top}\label{sec:obj2}

The previous category considers methods which minimize the false-negatives below the highest-ranked negative. Accuracy at the Top \cite{boyd2012accuracy} takes a different approach and minimizes false-positives above the top $\tau$-quantile defined by
\begin{equation}\label{eq:defin_quantile} 
t_{\rm Q}(\bm w) = \max\left\{t\mid \tp(\bm w,t) + \fp(\bm w,t)\ge n\tau\right\}.
\end{equation}
Then the Accuracy at the Top problem is defined by
\begin{equation}\label{eq:problem_aatp_orig}
\aligned
{\rm minimize} & \quad \frac{1}{n^{-}}\fp(\bm w,t) \\
\textnormal{subject to} & \quad t\text{ is the top \ensuremath{\tau}-quantile: it solves }\eqref{eq:defin_quantile}.
\endaligned
\end{equation}
Due to Lemma \ref{lemma:fnfp_equivalence} in the Appendix, we may equivalently (up to a small theoretical issue) replace the \eqref{eq:problem_aatp_orig} either by 
\begin{equation}\label{eq:problem_aatp_grill}
\aligned
{\rm minimize} & \quad \frac{1}{n^{+}}\fn(\bm w,t) + \frac{1}{n^{-}}\fp(\bm w,t)\\
\textnormal{subject to} & \quad t\text{ is the top \ensuremath{\tau}-quantile: it solves }\eqref{eq:defin_quantile}.
\endaligned
\end{equation}
or equivalently by
\begin{equation}\label{eq:problem_aatp}
\aligned
{\rm minimize} & \quad \frac{1}{n^{+}}\fn(\bm w,t)\\
\textnormal{subject to} & \quad t\text{ is the top \ensuremath{\tau}-quantile: it solves }\eqref{eq:defin_quantile}.
\endaligned
\end{equation}
While the former one falls fits into our framework \eqref{eq:problem1}, the latter one corresponds to the original definition from \cite{boyd2012accuracy} and makes a base for the \grill method described in the next paragraph.

\subsubsection{Grill}

The method from \cite{Grill_2016} builds on the Accuracy at the Top problem~\eqref{eq:problem_aatp_grill} where it replaces $\fn(\bm w, t)$ and $\fp(\bm w, t)$ in the objective by their surrogate counterparts $\fns(\bm w, t)$ and $\fps(\bm w, t)$. This leads to
\begin{equation}\label{eq:problem_grill}
\aligned
{\rm minimize} & \quad \frac{1}{n^{+}}\fns(\bm w,t) + \frac{1}{n^{-}}\fps(\bm w,t) + \frac\lambda2\norm{\bm w}^2\\
\textnormal{subject to} & \quad t\text{ is the top \ensuremath{\tau}-quantile: it solves }\eqref{eq:defin_quantile}.
\endaligned
\end{equation}
Based on the first author, we name this method \grill.

\subsubsection{\apatmat}

It is known that the quantile \eqref{eq:defin_quantile} is a non-convex function, which makes the optimization hard. Since \eqref{eq:defin_quantile} finds the threshold $t$ such that the fraction of samples above this threshold amounts $\tau$. This amounts to
\begin{equation}\label{eq:defin_quantile0}
\frac1n\sum_{\bm x\in \Xcal}\left[\beta(\bm w^\top \bm x-t)\right] = \tau,
\end{equation}
where $\beta$ is any positive scalar. This gives the idea to replace the counting function $[\cdot]$ by the surrogate function $l(\cdot)$ to arrive at the surrogate top $\tau$-quantile
\begin{equation}\label{eq:defin_quantile_surr} 
\bar{t}_{\rm Q}(\bm w)\quad \text{ solves }\quad \frac1n\sum_{\bm x\in \Xcal}l(\beta(\bm w^\top \bm x-t)) = \tau.
\end{equation}
As we will see later, the surrogate quantile \eqref{eq:defin_quantile_surr} is a convex approximation of the non-convex quantile \eqref{eq:defin_quantile}.

Then, we provide an alternative to \grill by replacing the true quantile by its surrogate counterpart and define propose problem
\begin{equation}\label{eq:problem_patmat}
\aligned
{\rm minimize} & \quad \frac{1}{n^{+}}\fns(\bm w,t) + \frac\lambda2\norm{\bm w}^2\\
\textnormal{subject to} & \quad t\text{ is the surrogate top \ensuremath{\tau}-quantile: it solves }\eqref{eq:defin_quantile_surr}.
\endaligned
\end{equation}
Note that \grill minimizes the convex combination of false-positives and false-negatives while \eqref{eq:problem_patmat} minimizes the false-negatives. The reason for this will be evident in Section \ref{sec:convexity} and amounts to preservation of convexity. Moreover, as will see later, problem \eqref{eq:problem_patmat} provides a good approximation to the Accuracy at the Top problem, it is easily solvable due to convexity and requires almost no tuning, we named it \patmat{} (Precision At the Top \& Mostly Automated Tuning).

\subsubsection{TopMean}

The main purpose of the surrogate quantile \eqref{eq:defin_quantile_surr} is to provide a convex approximation of the non-convex quantile \eqref{eq:defin_quantile}. We propose another convex approximation based again on \cite{lapin.2015}. If we take the $n\tau$ largest scores $z$ and 
compute their mean, we arrive at
\begin{equation}\label{eq:problem_topmeank}
\aligned
{\rm minimize} & \quad\frac{1}{n^{+}}\fns(\bm w,t) + \frac\lambda2\norm{\bm w}^2 \\
\textnormal{subject to} & \quad t = \frac 1{n\tau}(z_{[1]} + \dots + z_{[n\tau]}), \\
& \quad \text{components of }\bm z\text{ equal to }z=\bm w^\top \bm x\text{ for }\bm x \in \Xcal.
\endaligned
\end{equation}
There is a close connection between \toppushk and \topmeank. The former provides stability to \toppush and thus, the threshold is computed from negative scores and $k$ is small. On the other hand, the latter uses the threshold as an approximation of the quantile \eqref{eq:defin_quantile} and thus, the threshold is computed from all scores and $k=n\tau$ is large.

\subsection{Methods optimizing the Neyman-Pearson criterion}\label{sec:obj3}

Another category falling into the framework of \eqref{eq:problem1} and \eqref{eq:problem2} is the Neyman-Pearson problem which is closely related to hypothesis testing, where null $H_0$ and alternative $H_1$ hypotheses are given. Type~I error occurs when $H_0$ is true but is rejected and type II error happens when $H_0$ is false but it fails to be rejected. The standard technique is to minimize Type II error while a bound for Type I error is given.

In the Neyman-Pearson problem, the null hypothesis $H_0$ states that a sample $\bm x$ has the negative label. Then Type I error corresponds to false-positives while Type II error to false-negatives. If the bound on Type I error equals $\tau$, we may write this as
\begin{equation}\label{eq:defin_quantile_np} 
t_{\rm NP}(\bm w) = \max\left\{t\mid \fp(\bm w,t)\ge n^-\tau\right\}.
\end{equation}
Then, we may write the Neyman-Pearson problem as
\begin{equation}\label{eq:problem_np}
\aligned
{\rm minimize} & \quad \frac{1}{n^{+}}\fn(\bm w,t) \\
\textnormal{subject to} & \quad t\text{ is Type I error at level \ensuremath{\tau}: it solves }\eqref{eq:defin_quantile_np}.
\endaligned
\end{equation}
Since \eqref{eq:problem_np} differs from \eqref{eq:problem_aatp} only by counting only the false-positives in \eqref{eq:defin_quantile_np} instead of counting all positives in \eqref{eq:defin_quantile}, we can derive its three approximations in exactly the same way as in Section \ref{sec:obj2}. For this reason, we provide only their brief description.

\subsubsection{\npA}

Replacing the true counts by their surrogates results in the Neyman-Pearson variant of the \grill method
\begin{equation}\label{eq:problem_grill_np}
\aligned
{\rm minimize} & \quad \frac{1}{n^{+}}\fns(\bm w,t) + \frac{1}{n^{-}}\fps(\bm w,t) + \frac\lambda2\norm{\bm w}^2\\
\textnormal{subject to} & \quad t\text{ is the Neyman-Pearson threshold: it solves }\eqref{eq:defin_quantile_np}.
\endaligned
\end{equation}

\subsubsection{\npB}

Similarly as the surrogate quantile \eqref{eq:defin_quantile_surr} is a convex approximation of the non-convex quantile \eqref{eq:defin_quantile}, the surrogate Neyman-Pearson threshold
\begin{equation}\label{eq:defin_quantile_surr_np} 
\bar{t}_{\rm NP}(\bm w)\quad \text{ solves }\quad \frac1{n^-}\sum_{\bm x^-\in \Xcal^-}l(\beta(\bm w^\top \bm x^--t)) = \tau.
\end{equation}
is a convex approximation of the non-convex Neyman-Pearson threshold \eqref{eq:defin_quantile_np}. Then, similarly to \patmat, we write its Neyman-Pearson variant
\begin{equation}\label{eq:problem_patmat_np}
\aligned
{\rm minimize} & \quad \frac{1}{n^{+}}\fns(\bm w,t) + \frac\lambda2\norm{\bm w}^2\\
\textnormal{subject to} & \quad t\text{ is the surrogate Neyman-Pearson threshold: it solves }\eqref{eq:defin_quantile_surr_np}.
\endaligned
\end{equation}

\subsubsection{\npC}

Finally, the Neyman-Pearson alternative to \topmeank reads
\begin{equation}\label{eq:problem_topmeank_np}
\aligned
{\rm minimize} & \quad\frac{1}{n^{+}}\fns(\bm w,t) + \frac\lambda2\norm{\bm w}^2 \\
\textnormal{subject to} & \quad t = \frac 1{n^-\tau}(z_{[1]}^- + \dots + z_{[n^-\tau]}^-), \\
& \quad \text{components of }\bm z^-\text{ equal to }z^-=\bm w^\top \bm x^-\text{ for }\bm x^- \in \Xcal.
\endaligned
\end{equation}
We may see this problem in two different viewpoints. First, \npC provides a convex approximation of \npA. Second, \npC has the same form as \toppushk. The only difference is that for \npC we have $k=n^-\tau$ while for \toppushk the value of $k$ is small. Thus, even though we started from two different problems, we arrived at two approximations which differ only in the value of one parameter. This shows a close relation of the ranking problem and the Neyman-Pearson problem and the need for a unified theory to handle these problems.

\section{Example of a Degenerate Behavior}\label{sec:example}

In the previous section, we presented multiple criteria and methods. In this section, we provide a simple example and show that the state of the art method \toppush degenerates for it. This example also provides a motivation for the extensive theoretical analysis in Section~\ref{sec:theory} which is summarized in Table \ref{tab:methods} on page \pageref{tab:methods}, 

To define this example, consider the case of $n$ negative samples uniformly distributed in $[-1,0]\times[-1,1]$, $n$ positive samples uniformly distributed in $[0,1]\times[-1,1]$ and one negative sample at $(2,0)$, see Figure \ref{fig:example} (left). If $n$ is large, the point at $(2,0)$ is an outlier and the dataset is perfectly separable. The separating hyperplane has the normal vector $\bm w=(1,0)$. Since the methods for the Neyman-Pearson problem are similar to those for the Accuracy at the Top problem, we consider only the five methods described in Sections \ref{sec:obj1} and \ref{sec:obj2} with the hinge loss and no regularization. We provide only the results while postponing the precise computation into Appendix \ref{app:example}.

\begin{figure}[!ht]
\centering
\includegraphics[width=0.9\linewidth]{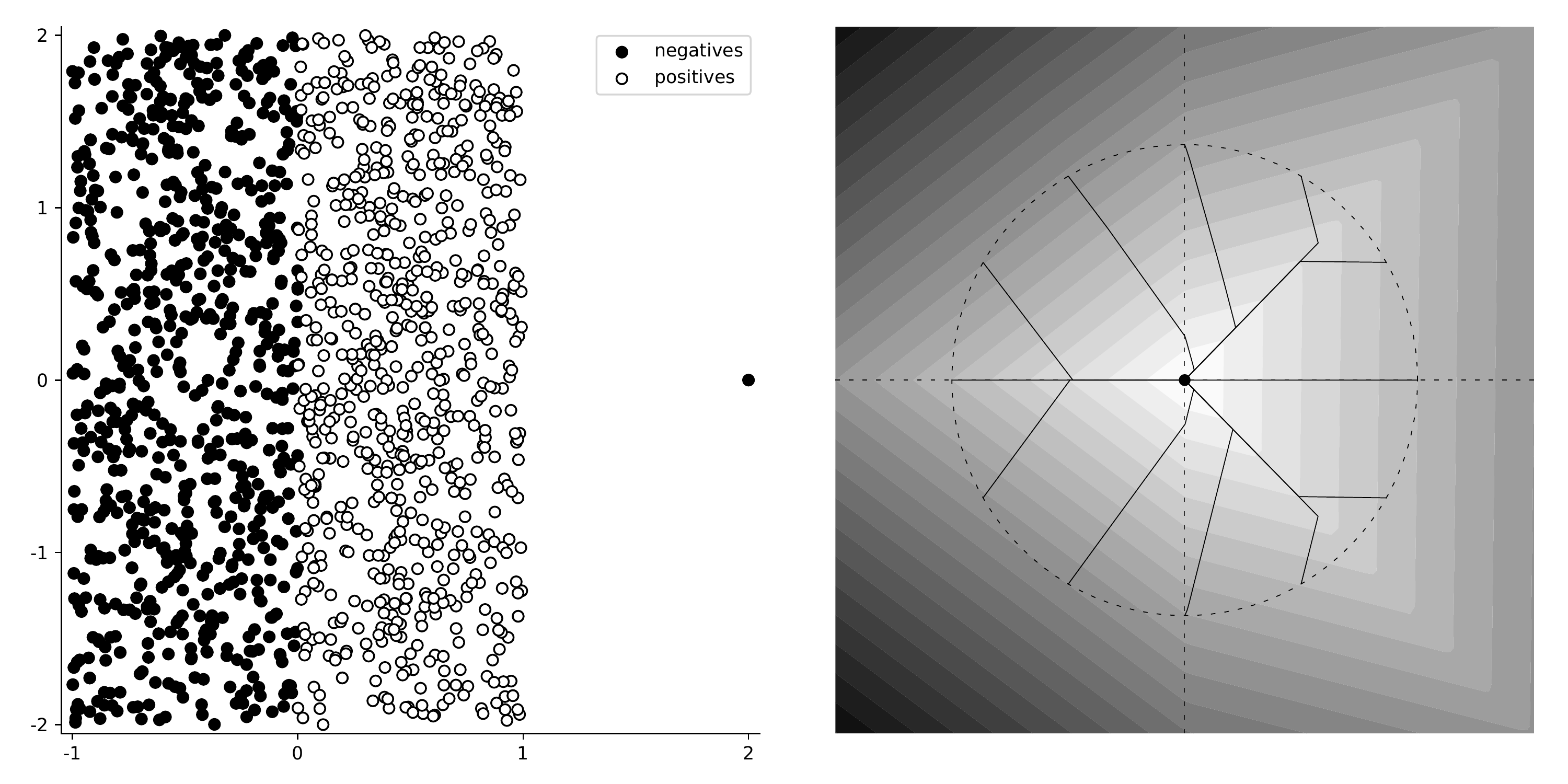}
\caption{Left: distribution of positive (empty circle) and negative samples (full circles) for the example from Section \ref{sec:example}. Right: contour plot for \toppush and its convergence to zero from $12$ initial points.}
\label{fig:example}
\end{figure}

\begin{comment}
\begin{figure}
\begin{minipage}{0.49\linewidth}
\includegraphics[width=\linewidth]{Figures/Example.png}
\end{minipage}
\begin{minipage}{0.49\linewidth}
\includegraphics[width=\linewidth]{Figures/toppush_convergence.pdf}
\end{minipage}
\caption{Left: distribution of positive (empty circle) and negative samples (full circles) for the example from Section \ref{sec:example}. Right: contour plot for \toppush and its convergence to zero from $12$ initial points.}
\label{fig:example}
\end{figure}
\end{comment}

Consider first solution $\bm w_1=(0,0)$. It is obvious that $\bm w_1^\top \bm x=0$ for all $\bm x$ and thus for the threshold we have $t=0$ for all methods with the exception of \patmat where $t=\frac{1}{\beta}(1-\tau)$. The objective then equals to $f(\bm w_1)=1+t$ for all method with the exception of \grill where we need to add the false-positives to obtain $f(\bm w_1)=2$.

Consider now solution $\bm w_2=(1,0)$. Since the \toppush method chooses the largest negative, it sets $t=2$. \toppushk chooses the $k^{\rm th}$ largest negative and sets $t=\frac2k$. \grill selects the $\tau$-top quantile, which for the uniform distribution on the interval $[-1,1]$ equals to $t=1-2\tau$. For the \patmat method, it can be computed that for $\beta\le\tau$ we have $t=\frac{1}{\beta}(1-\tau)$ and finally \topmeank computes the average between the true quantile $1-2\tau$ and the upper bound $1$ and thus $t=1-\tau$. The objective then equals to $f(\bm w_2)=0.5+t$ for all method with the exception of \grill where we have to add the false-positives to get $f(\bm w_2)=0.5+t+0.5(1-t)^2$.

These results are summarized in Table \ref{tab:example}. We chose these two points because both are important: $\bm w_1$ does not generate any separating hyperplane while $\bm w_2$ is the normal vector to the optimal separating hyperplane. Since the dataset is perfectly separable by $\bm w_2$, we expect that $\bm w_2$ provides a lower objective than $\bm w_1$. By shading the better objective in Table~\ref{tab:example} by grey, we see that this did not happen for \toppush and \topmeank.

\begin{table}[!ht]
\caption{Comparison of methods on the very simple problem from Section \ref{sec:example}. Two methods have the global minimum at $\bm w_1=(0,0)$ which does not determine any separating hyperplane. The perfect separating hyperplane is generated by $\bm w_2=(1,0)$.}
\label{tab:example}
\centering
\begin{tabular}{@{}l ccccc@{}}\toprule
& & \multicolumn{2}{c}{$\bm w_1=(0,0)$} & \multicolumn{2}{c}{$\bm w_2=(1,0)$} \\ \cmidrule(lr){3-4} \cmidrule(lr){5-6}
Name & Label & $t$& $f$ & $t$ & $f$ \\
\midrule
\atoppush & \eqref{eq:problem_toppush} & $0$ & \cellcolor{gray!50}$1$ & $2$ & $2.5$ \\
\atoppushk & \eqref{eq:problem_toppushk} & $0$ & $1$ & $\frac2k$ & \cellcolor{gray!50}$0.5+\frac2k$ \\
\agrill & \eqref{eq:problem_grill} & $0$ & $2$ & $1-2\tau$ & \cellcolor{gray!50}$1.5+2\tau(1-\tau)$ \\
\apatmat & \eqref{eq:problem_patmat}  & $\frac{1}{\beta}(1-\tau)$ & $1+\frac{1}{\beta}(1-\tau)$ & $\frac{1}{\beta}(1-\tau)$ & \cellcolor{gray!50}$0.5+\frac{1}{\beta}(1-\tau)$ \\
\atopmeank & \eqref{eq:problem_topmeank} & $0$ & \cellcolor{gray!50}$1$ & $1-\tau$ & $1.5-\tau$ \\
\bottomrule
\end{tabular}
\end{table}

It can be shown that $\bm w_1=(0,0)$ is even the global minimum for \toppush and \topmeank. This raises the question of whether some tricks, such as early stopping or excluding a small ball around zero, cannot overcome this difficulty. The answer is negative as shown in Figure \ref{fig:example} (right). Here, we run the \toppush method from several starting points and it always converges to zero from one of the three possible directions; all of them far from the normal vector to the separating hyperplane. This shows the need for a formal analysis of the framework proposed in Section \ref{sec:framework}.

\section{Theoretical Analysis of the Framework}\label{sec:theory}

In this section, we provide a theoretical analysis of the unified framework from Section \ref{sec:framework}. We focus mainly on the desirable property of convexity and the undesirable feature of having the global minimum at $\bm w=0$. Note that convexity makes solving the problem much easier while $\bm w=0$ does not generate a sensible solution. The results are summarized in Table \ref{tab:methods} below. The proofs are postponed to Appendix \ref{app:proofs}.

\subsection{Convexity}\label{sec:convexity}

Convexity is one of the most important properties in numerical optimization. It ensures that the optimization problem has neither stationary points nor local minima. All points of interest are global minima. Moreover, many optimization algorithms have guaranteed convergence or faster convergence rates in the convex setting \cite{boyd.2004}.

\begin{theorem}\label{thm:convex}
If the threshold $t$ is a convex function of the weights $\bm w$, then function $f(\bm w) = \fns(\bm w, t(\bm w))$ is convex.
\end{theorem}

This theorem allows us to reduce the analysis of the convexity of the methods described above just to the $\bm w\mapsto t$ function. Due to \cite{lapin.2015}, we know that the mean of the $k$ highest values of a vector is a convex function. This immediately implies that \toppush, \toppushk, \topmeank and \npC are convex problems. The convexity of \patmat and \npB follows from the theory of convex function summarized in Lemma \ref{lemma:convex_surr} in the Appendix. At the same time, \grill and \npA are not convex problems. However, they are at least continuous which we prove in Lemma \ref{lemma:quantile_continuous}.

\subsection{Robustness and Global minimum at zero}\label{sec:w_zero}

The convexity derived in the previous section guarantees that there are no local minima. However, as we have shown in Section \ref{sec:example}, it may happen that the global minimum is at $\bm w=0$. This is a highly undesirable situation since $\bm w$ is the normal vector to the separating hyperplane and the zero vector provides no information. In this section, we analyze when this situation happens. Recall that the threshold $t$ depends on the weights $\bm w$; sometimes we stress this by writing $t(\bm w)$.

The first result states that if the threshold $t(\bm w)$ is above a certain value, then zero has a better objective that $\bm w$. If this happens for all $\bm w$, then zero is the global minimum.

\begin{theorem}\label{thm:large_t}
Consider any of these methods: \toppush, \toppushk, \topmeank or \npC. Fix any $\bm w$ and denote the corresponding threshold $t(\bm w)$. If we have
\begin{equation}\label{eq:w_zero}
t(\bm w)\ge \frac{1}{n^+} \sum_{\bm x^+\in\Xcal^+} \bm w^\top \bm x^+,
\end{equation}
then $f(\bm 0)\le f(\bm w)$.
\end{theorem}

We can use this result immediately to deduce that some methods may have the global minimum at $\bm w=0$. More specifically, \toppush fails if there are outliers (and \toppushk fails if there are many outliers) and \topmeank fails whenever there are many positive samples.

\begin{corollary}\label{cor:toppush}
Consider the \toppush method. If the positive samples lie in the convex hull of negative samples, then $\bm w=0$ is the global minimum.
\end{corollary}

\begin{corollary}\label{cor:topmean}
Consider the \topmeank method. If $n^+\ge n\tau$, then $\bm w=0$ is the global minimum.
\end{corollary}

The proof of Theorem \ref{thm:large_t} employs the fact that all methods in the theorem statement have only false-negatives in the objective. If $\bm w_1=0$, then $\bm w_1^\top \bm x=0$ for all samples $\bm x$, the threshold equals to $t=0$ and the objective equals to one. If the threshold is large for some $\bm w$, many positives are below the threshold and the false-negatives have the average surrogate value larger than one. In such a case, $\bm w=0$ becomes the global minimum.

There are two fixes to this situation. The first one is to include false-positives in the objective. The second one is to move the threshold from zero even when all scores $\bm w^\top \bm x$ equal to zero. The first approach was taken by \grill and \npA and necessarily results in the loss of convexity. The other approach was taken by our methods \patmat and \npB. It keeps the convexity and, as we derive in the next result, the global minimum is away from zero.

\begin{theorem}\label{thm:patmat_zero}
Consider the \patmat or \npB method with the hinge surrogate and no regularization. Assume that for some $\bm w$ we have
\begin{equation}\label{eq:patmat_zero}
\frac{1}{n^+}\sum_{\bm x^+\in \Xcal^+}\bm w^\top \bm x^+ > \frac{1}{n^-}\sum_{\bm x^-\in \Xcal^-}\bm w^\top \bm x^-.
\end{equation}
Then there exists scaling parameter $\beta_0$ from \eqref{eq:defin_quantile0} such that for all $\beta\in(0,\beta_0)$ we have $f(\bm w)<f(\bm 0)$.
\end{theorem}

We would like to compare the results of Theorems \ref{thm:large_t} and \ref{thm:patmat_zero}. The former states that if the average score $\bm w^\top \bm x^-$ for a \emph{small number of largest} negative samples is larger than the average score $\bm w^\top \bm x^+$ for \emph{all} positive samples, then the corresponding method fails. We show in Lemma \ref{lemma:bound} in the Appendix that this number equals to $1$ for \toppush, to $k$ for \toppushk and to $n^-\tau$ for \npC. On the other hand, Theorem \ref{thm:patmat_zero} states that if the average score $\bm w^\top \bm x^+$ for \emph{all} positive samples is larger than the average score $\bm w^\top \bm x^-$ for \emph{all} negative samples, then \patmat and \npB do not have problems with zero. Note that since we push positives to the top, for a good classifier $\bm w^\top \bm x^+$ should be higher than $\bm w^\top \bm x^-$.

\subsection{Threshold comparison}

Theorem \ref{thm:large_t} and the last paragraph of Section \ref{sec:w_zero} show that the larger the threshold $t$ for a method, the more inclined it is to have the global minimum at $\bm w=0$. This gives rise to the necessity to compare the threshold values.

The threshold is computed from scores $z=\bm w^\top \bm x$. Since the threshold for \toppush, \toppushk and \npC is the average of $1$, $k$ and $n^-\tau$, respectively, of largest scores corresponding to negative samples, and since $k\ll n^-\tau$ the threshold for \toppush is greater than the one for \toppushk which is in turn greater than the one for \npC. Moreover, the threshold for \patmat is greater than the threshold for \topmeank due to Lemma \ref{lemma:thresholds1} which in turn is greater than the threshold for \grill due to its definition. This holds true for their Neyman-Pearson counterparts as well. For good classifiers, the thresholds for the Neyman-Pearson methods from Section \ref{sec:obj3} are smaller than their Accuracy at the Top counterparts from Section \ref{sec:obj2}; for a precise formulation see Lemma \ref{lemma:thresholds2}.

\subsection{Method comparison}

We provide a visualization of the obtained results in Table \ref{tab:methods} and Figure \ref{fig:thresholds}. Table \ref{tab:methods} gives the basic characterization of the methods such as their definition label, their source, the criterion they approximate, the hyperparameters, whether the method is convex and whether it has problems with $\bm w=0$ based on Theorem \ref{thm:large_t}. 

%
%
%
%
%
% TO THE PUBLISHER
%
%
%
%
% Sorry, there are hardcoded footnotes in the table. But I spend one hour how to get and reference footnotes in the floating environment and then I gave up.
%
%
%
%

\begin{table}[!ht]
\caption{Summary of the methods from Section \ref{sec:framework}. The table shows their definition label, the source or the source they are based on, the criterion they approximate, the hyperparameters, whether the method is convex and whether the method is robust (in the sense of having problems with $\bm w=0$).}
\label{tab:methods}
\centering
\begin{tabular}{>{\kern-0.5\tabcolsep}ll ccccc<{\kern-0.5\tabcolsep}}\toprule
Name & Source & Definition & Criterion & Hyperpars & Convex & Robust \\
\midrule
 \atoppush & \cite{Li_TopPush} & \eqref{eq:problem_toppush} & \eqref{eq:problem_top} & $\lambda$ & \cmark & \xmark\\
\atoppushk & Ours\textsuperscript{2} & \eqref{eq:problem_toppushk} & \eqref{eq:problem_top} & $\lambda$, $k$ & \cmark & \xmark\\ \headcol
 \agrill & \cite{Grill_2016} & \eqref{eq:problem_grill}  & \eqref{eq:problem_aatp_orig} & $\lambda$ & \xmark & \cmark\\ \headcol
\apatmat & Ours\textsuperscript{2} & \eqref{eq:problem_patmat} & \eqref{eq:problem_aatp_orig} & $\beta$, $\lambda$ & \cmark & \cmark\\ \headcol
 \atopmeank & Ours\textsuperscript{2} & \eqref{eq:problem_topmeank} & \eqref{eq:problem_aatp_orig} & $\lambda$ & \cmark & \xmark\\ 
\anpA & Ours\textsuperscript{2} & \eqref{eq:problem_grill_np} & \eqref{eq:problem_np} & $\lambda$ & \xmark & \cmark\\
 \anpB & Ours\textsuperscript{2} & \eqref{eq:problem_patmat_np} & \eqref{eq:problem_np} & $\beta$, $\lambda$ & \cmark & \cmark\\
\anpC & Ours\textsuperscript{2} & \eqref{eq:problem_topmeank_np} & \eqref{eq:problem_np} & $\lambda$ & \cmark & \xmark\\
\bottomrule
\end{tabular}
\end{table}

\tikzstyle{bubbleA} = [rectangle, rounded corners, minimum width=2cm, minimum height=1cm,text centered, draw=black]
\tikzstyle{bubbleB} = [bubbleA, fill=gray!40]
\tikzstyle{bubbleC} = [bubbleA, fill=gray!125]
\tikzstyle{arrowA} = [thick,->,>=stealth]
\tikzstyle{arrowB} = [thick,dotted,->,>=stealth]

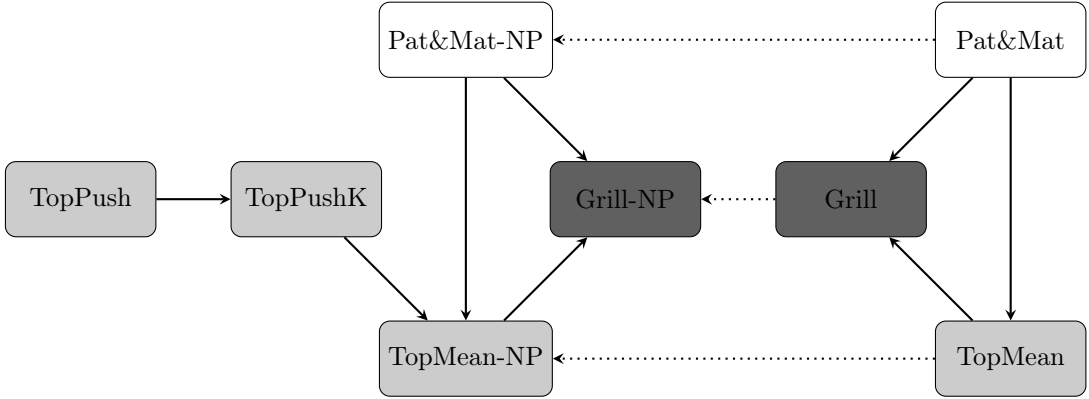
\begin{figure}[!ht]
\centering
\begin{tikzpicture}[node distance=3cm]
\node (tikz1) [bubbleB] {\atoppush};
\node (tikz2) [bubbleB, right of=tikz1] {\atoppushk};
\node (tikz3) [bubbleA, above right of=tikz2] {\anpB};
\node (tikz4) [bubbleB, below right of=tikz2] {\anpC};
\node (tikz5) [bubbleC, below right of=tikz3] {\anpA};
\draw [arrowA] (tikz1) -- (tikz2);
\draw [arrowA] (tikz2) -- (tikz4);
\draw [arrowA] (tikz3) -- (tikz5);
\draw [arrowA] (tikz4) -- (tikz5);
\draw [arrowA] (tikz3) -- (tikz4);
\node (tikz8) [bubbleC, right of=tikz5] {\agrill};
\node (tikz6) [bubbleA, above right of=tikz8] {\apatmat};
\node (tikz7) [bubbleB, below right of=tikz8] {\atopmeank};
\draw [arrowA] (tikz6) -- (tikz8);
\draw [arrowA] (tikz7) -- (tikz8);
\draw [arrowA] (tikz6) -- (tikz7);
\draw [arrowB] (tikz6) -- (tikz3);
\draw [arrowB] (tikz7) -- (tikz4);
\draw [arrowB] (tikz8) -- (tikz5);
\end{tikzpicture}
\caption{Summary of the methods from Section \ref{sec:framework}. Methods in white and light gray are convex while methods in dark gray are non-convex. Methods in light gray are vulnerable to have the global minimum at $\bm w=0$. Full (dotted) arrow pointing from one method to another show that the latter method has always (usually) smaller threshold.}
\label{fig:thresholds}
\end{figure}

A similar comparison is performed in Figure \ref{fig:thresholds}. Methods in white and light grey are convex while methods in dark grey are non-convex.\footnote{\toppushk is a modification of \cite{Li_TopPush}, \patmat and \npB are loosely connected to \cite{rockafellar2000optimization}, \topmeank and \npC are inspired by \cite{lapin.2015} and \npA is a modification of \cite{Grill_2016}.\label{foot:methods}} Based on Theorem \ref{thm:large_t}, four methods in light grey are vulnerable to have the global minimum at $\bm w=0$. This theorem states that the higher the threshold, the more vulnerable the method is. This dependence is depicted by the full arrows. If it points from one method to another, the latter one has a smaller threshold and thus is less vulnerable to this undesired global minima. The dotted arrows indicate that this holds true usually but not always. This complies with Corollaries \ref{cor:toppush} and \ref{cor:topmean} which state that \toppush and \topmeank are most vulnerable.

\section{Numerical Considerations}\label{sec:num1}

So far, we have performed a theoretical analysis of the presented methods. In this section, we present some numerical considerations as well.

\subsection{Gradient computation and Variable reduction}

The decision variables in \eqref{eq:problem2} are the normal vector of the separating hyperplane $\bm w$ and the threshold $t$. To apply an efficient optimization method, we need to compute gradients. The simplest idea is to perform the gradient only with respect to $\bm w$ and then recompute $t$. This was followed in \cite{Grill_2016}.

However, there is a more sophisticated way based on the implicit function theorem. For each $\bm w$, the threshold $t$ can be computed uniquely. We stress this dependence by writing $t(\bm w)$ instead of $t$. By doing so, we effectively remove the threshold $t$ from the decision variables and $\bm w$ remains the only decision variable. Note that the convexity is preserved. Then we can compute the derivative via the chain rule
\begin{equation}\label{eq:derivatives}
\aligned
f(\bm w) &= \frac{1}{n^+}\sum_{\bm x\in\Xcal^+} l(t(\bm w) - \bm w^\top \bm x) + \frac{\lambda}{2}\norm{\bm w}^2, \\
\nabla f(\bm w) &= \frac{1}{n^+}\sum_{\bm x\in\Xcal^+} l'(t(\bm w) - \bm w^\top \bm x)(\nabla t(\bm w) - \bm x) + \lambda \bm w.
\endaligned
\end{equation}
The only remaining part is the computation of $\nabla t(w)$. We show an efficient computation in Appendix \ref{app:derivatives}.

\subsection{Stochastic gradient descent}

From the objective and gradient evaluation in \eqref{eq:derivatives}, it is evident that we need to evaluate $\bm w^\top \bm x^+$ for all $\bm x^+\in\Xcal^+$. At the same time, to compute the threshold, all methods require to evaluate $\bm w^\top \bm x^-$ for all $\bm x^-\in\Xcal^-$. Thus, at every iteration, we need to evaluate $\bm w^\top \bm x$ for all samples $\bm x$. Even though some heuristics %, such as updating the index with the largest score $z_{[1]}^-$ in \toppush only every few iterations,
are possible, we decided to employ the standard technique of stochastic gradient descent where the dataset is randomly divided into several minibatches and the threshold $t$, objective $f(\bm w)$ and its gradient $\nabla f(\bm w)$ are updated only on the current minibatch.

\subsection{Performance criteria}

In Section \ref{sec:framework}, we described three classes of methods, each optimizing a different criterion. Utilizing the ``Criterion'' column in Table \ref{tab:methods}, we summarize these three criteria in Table \ref{tab:criteria}. We recall that the precision and recall are for a threshold $t$ defined by
\begin{equation}\label{eq:prec_rec}
{\rm Precision}=\frac{\tp(\bm w,t)}{\tp(\bm w,t)+\fp(\bm w,t)},\quad{\rm Recall}=\frac{\tp(\bm w,t)}{\tp(\bm w,t)+\fn(\bm w,t)}.
\end{equation}
We will show the Precision-Recall (PR) and the Precision-$\tau$ (P$\tau$) curves. The former is a well-accepted visualization for highly unbalanced data \cite{davis2006relationship} while the latter better reflects that the methods concentrate only on the top of the scores.

\begin{table}[!ht]
\caption{The criteria for the three classes from Section \ref{sec:framework}.}
\label{tab:criteria}
\centering
\begin{tabular}{@{}l l@{}}\toprule
Name & Description \\ \midrule
Positives@Top & Fraction of positives above the largest negative \eqref{eq:problem_top} \\
Positives@Quantile & Fraction of positives above the $\tau$-quantile \eqref{eq:problem_aatp_orig} \\
Positives@NP & Fraction of positives above the Neyman-Pearson threshold \eqref{eq:problem_np} \\
\bottomrule
\end{tabular}
\end{table}

\subsection{Implementational details and Hyperparameter choice}

We recall that all methods fall into the framework of either \eqref{eq:problem1} or \eqref{eq:problem2}. Since the threshold $t$ depends on the weights $\bm w$, we can consider the decision variable to be only $\bm w$. Then to apply a method, we implemented the following iterative procedure. At iteration $j$, we have the weights $\bm w^j$ to which we compute the threshold $t^j=t(\bm w^j)$. Then according to \eqref{eq:derivatives} we compute the gradient of the objective and apply the ADAM descent scheme \cite{kingma2014adam}. All methods were run for $1000$ iterations. Moreover, for large datasets, we replaced the standard gradient descent by its stochastic counterpart \cite{bottou.2010} where the threshold $t^j$ and the gradient are computed only on a part of the dataset called a minibatch. All methods used the hinge surrogate \eqref{eq:defin_surrogate}.

We run the methods for the following hyperparameters
\begin{equation}\label{eq:beta1}
\aligned
\beta &\in  \left\{0.0001,\ 0.001,\ 0.01,\ 0.1,\ 1,\ 10\right\}, \\
\lambda &\in \left\{0,\ 0.00001,\ 0.0001,\ 0.001,\ 0.01,\ 0.1\right\}, \\
k &\in \left\{1,3,5,10,15,20\right\}.
\endaligned
\end{equation}
For \toppushk, \patmat and \npB we fixed $\lambda=0.001$ to have six hyperparameters for all methods. For all datasets, we choose the hyperparameter which minimized the corresponding criterion from Table \ref{tab:methods} on the validation set. The results are computed on the testing set which was not used during training the methods.

Note that \toppush was originally implemented in the dual. However, to allow for the same framework and for the stochastic gradient descent, we implemented it in the primal. These two approaches are, at least theoretically, equivalent. For \grill and \npA we stick to the original paper and apply the projection of weights $\bm w$ onto the $l_2$-unit ball after each gradient step. This is not advised for other methods as convexity would be lost.

\subsection{Dataset description}

For the numerical results, we considered nine datasets summarized in Table \ref{tab:counts}. Five of them are standard and can be downloaded from the UCI repository. Datasets Ionosphere \cite{ionosphere} and Spambase are small, Hepmass \cite{hepmass} contains a large number of samples while Gisette \cite{gisette} contains a large number of features. To make the dataset more difficult, we created Hepmass 10\% dataset which is a random subset of the Hepmass dataset and reduces the fraction of positive samples from 50\% to 10\%.

Besides these datasets, we also considered two real-world datasets CTA and NetFlow obtained from existing network intrusion detection systems. Following \cite{Grill_2016}, we considered both their vanilla variant and artificially introduced noise MLT by removing positive samples from the training sets but keeping them in the testing set. This simulates the situation that analysts failed to identify some malware. This makes the dataset difficult since the basic assumption of training and testing set having the same distribution is violated. More description about these datasets can be found in Appendix \ref{sec:cisco}.

\begin{table}[!ht]
\caption{Structure of the used datasets. The training, validation and testing sets show the number of features $m$, minibatches $n_{\rm minibatch}$, samples $n$ and the fraction of positive samples $\frac{n^+}{n}$.}
\label{tab:counts}
\centering
\begin{tabular}{@{}lllllllll@{}}
\toprule
 &  & & \multicolumn{2}{c}{Training} & \multicolumn{2}{c}{Validation} & \multicolumn{2}{c}{Testing} \\ \cmidrule(lr){4-5} \cmidrule(lr){6-7} \cmidrule(l){8-9} 
 & $m$ & $n_{\rm minibatch}$ & $n$ & $\frac{n^+}{n}$ & $n$ & $\frac{n^+}{n}$ & $n$ & $\frac{n^+}{n}$ \\ \midrule
Ionosphere & $34$ & $1$ & $175$ & $36.0\%$ & $88$ & $35.2\%$ & $88$ & $36.4\%$ \\
Spambase & $57$ & $1$ & $2300$ & $39.4\%$ & $1150$ & $39.4\%$ & $1151$ & $39.4\%$ \\
Gisette & $5001$ & $1$ & $6000$ & $50.0\%$ & $500$ & $50.0\%$ & $500$ & $50.0\%$ \\
Hepmass & $28$ & $40$ & $5250000$ & $50.0\%$ & $2625040$ & $50.0\%$ & $2624960$ & $50.0\%$ \\
Hepmass 10 & $28$ & $40$ & $2916636$ & $10.0\%$ & $1458320$ & $10.0\%$ & $1458240$ & $10.0\%$ \\
CTA & $34$ & $26$ & $263796$ & $2.1\%$ & $263796$ & $2.2\%$ & $527488$ & $2.1\%$ \\
CTA MLT & $34$ & $9$ & $91314$ & $0.3\%$ & $91314$ & $0.3\%$ & $182592$ & $2.1\%$ \\
NetFlow & $25$ & $8$ & $142676$ & $8.7\%$ & $142674$ & $8.7\%$ & $285454$ & $8.6\%$ \\
NetFlow MLT & $25$ & $8$ & $145260$ & $2.7\%$ & $145261$ & $2.8\%$ & $285508$ & $8.6\%$ \\
\bottomrule
\end{tabular}
\end{table}

\section{Numerical experiments}\label{sec:num2}

In this section, we present the numerical results. Note that all results are computed on the testing set which was not available during training. We run the algorithms from $\bm w^0=0$ and from a randomly generated point in the interval $[-1,1]$. Since the former showed a better performance in $58.7\%$ cases, we show only the results starting from zero.
\begin{figure}
\begin{tikzpicture}
  \pgfplotsset{}
  \begin{groupplot}[group style = {group size = 2 by 3, horizontal sep = 5pt, vertical sep=5pt}, grid=major, grid style={dashed, gray!50}, ymin=0, ymax=1.05, ytick={0,0.2,0.4,0.6,0.8,1}]
      \nextgroupplot[col1, xticklabels={}, ylabel={Precision}, legend style = legendStyleA]
          \addplot [line1] table[x index=0, y index=1] {\tabAA}; \addlegendentry{\atoppushk}
          \addplot [line2] table[x index=0, y index=2] {\tabAA};   \addlegendentry{\apatmat}
          \addplot [line3] table[x index=0, y index=3] {\tabAA}; \addlegendentry{\anpB}
          \addplot [line4] table[x index=0, y index=4] {\tabAA};   \addlegendentry{\anpC}%
      \nextgroupplot[col3, xticklabels={}, yticklabels={}]
          \addplot [line1] table[x index=0, y index=1] {\tabAB};
          \addplot [line2] table[x index=0, y index=2] {\tabAB};
          \addplot [line3] table[x index=0, y index=3] {\tabAB}; 
          \addplot [line4] table[x index=0, y index=4] {\tabAB};            
      \nextgroupplot[col1, xticklabels={}, ylabel={Precision}] 
          \addplot [line1] table[x index=0, y index=1] {\tabBA}; 
          \addplot [line2] table[x index=0, y index=2] {\tabBA}; 
          \addplot [line3] table[x index=0, y index=3] {\tabBA}; 
          \addplot [line4] table[x index=0, y index=4] {\tabBA};   
      \nextgroupplot[col3, xticklabels={}, yticklabels={}] 
          \addplot [line1] table[x index=0, y index=1] {\tabBB}; 
          \addplot [line2] table[x index=0, y index=2] {\tabBB}; 
          \addplot [line3] table[x index=0, y index=3] {\tabBB}; 
          \addplot [line4] table[x index=0, y index=4] {\tabBB};   
      \nextgroupplot[col1, xlabel={$\tau$}, ylabel={Precision}]
          \addplot [line1] table[x index=0, y index=1] {\tabCA};
          \addplot [line2] table[x index=0, y index=2] {\tabCA};
          \addplot [line3] table[x index=0, y index=3] {\tabCA}; 
          \addplot [line4] table[x index=0, y index=4] {\tabCA};   
      \nextgroupplot[col3, xlabel={$\tau$}, yticklabels={}]
          \addplot [line1] table[x index=0, y index=1] {\tabCB};
          \addplot [line2] table[x index=0, y index=2] {\tabCB};
          \addplot [line3] table[x index=0, y index=3] {\tabCB}; 
          \addplot [line4] table[x index=0, y index=4] {\tabCB};   
      \end{groupplot}
\node at ($(group c2r1) + (-3.5,3.3)$) {\ref{grouplegendA}}; 
\end{tikzpicture}
\caption{P$\tau$ curves for datasets NetFlow (top row), Gisette (middle row) and Hepmass (bottom row) and two different zooms of $\tau\in[0,1]$ and $\tau\in[0,0.01]$ (columns).}
\label{fig:ptau}
\end{figure}
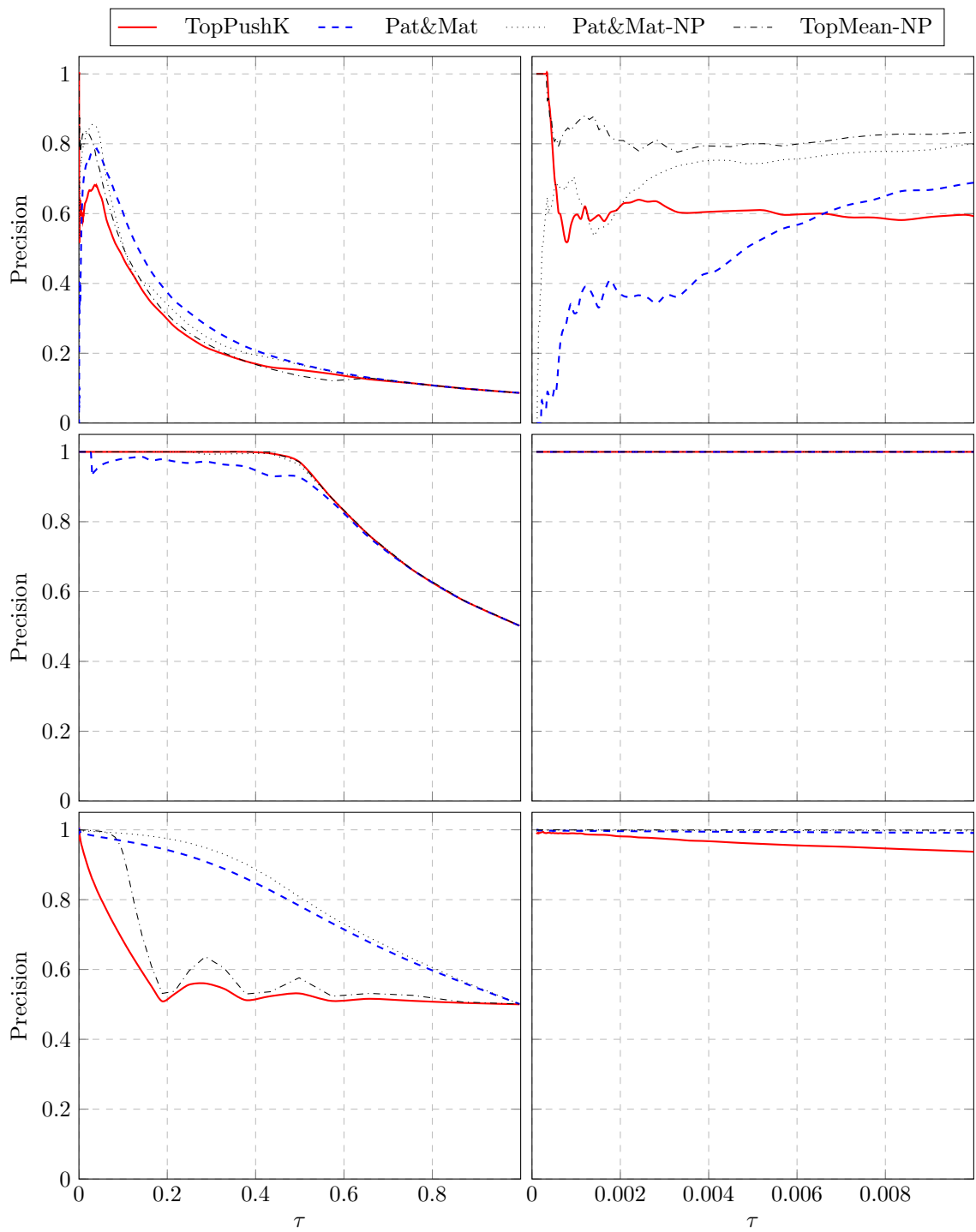

In Figure \ref{fig:ptau} we present the P$\tau$ curves which show the precision at the top $\tau$-quantile. They are equivalent to the well-known PR curves. Each row corresponds to one dataset (NetFlow top, Gisette middle and Hepmass bottom row), the left column shows the whole $\tau\in[0,1]$ while the right one is its zoomed version to $\tau\in[0,0.01]$. We show the \toppushk, \patmat, \npB and \npC methods which, as we will see later, perform the best. Note that the right column is more important as all these method focus on maximizing the accuracy only on the top of the dataset, which precisely corresponds to small values of $\tau$. The \patmat and \npB methods perform the best on larger quantiles. \toppushk performs reasonably well for large quantiles while its performance increases with decreasing $\tau$. In most cases, the precision is $1$ when $\tau$ is sufficiently small.

Having shown a good performance of the methods, we focus on their comparison. Since all methods optimize one of the criteria from Table \ref{tab:criteria}, we base the comparison on these criteria. First, we consider 14 methods (we count different values of $\tau$ as a different method) as depicted in Table \ref{tab:results2}. Note that \toppushk for $k=1$ reduces to the classical \toppush. For each dataset and each method, we evaluated criteria from Table \ref{tab:criteria}. For each dataset and criterion, we computed the rank of all methods and averaged them with respect to all datasets. Rank 1 refers to the best performance for given criteria while rank 14 is the worst. Since criteria are in columns, the comparison among methods is column-wise. The best method is depicted in dark grey and comparable methods (at most one rank away) are depicted in light grey.

\begin{table}[!ht]
\caption{The average rank of all methods across all datasets on the criteria from Table \ref{tab:criteria}. Rank 1 is the best and rank 14. This evaluation is performed for each criterion, thus in columns. Dark gray depicts the best method while light gray depicts comparable methods.}
\label{tab:results2}
\centering
\begin{tabular}{@{}llllll@{}}
\toprule
 & Positives@Top & \multicolumn{2}{c}{Positives@Quantile} & \multicolumn{2}{c}{Positives@NP} \\ \cmidrule(r){3-4} \cmidrule(l){5-6}
 & & $\tau=0.01$& $\tau=0.03$ & $\tau=0.01$ & $\tau=0.03$ \\
\midrule
\atoppush & $5.5$ & $8.1$ & $8.4$ & $8.3$ & $9.3$ \\
\atoppushk & \cellcolor{gray!60}$4.2$ & $7.2$ & $8.4$ & $7.4$ & $7.3$ \\
\agrill$\tau=0.01$ & $10.1$ & $9.7$ & $11.2$ & $11.4$ & $10.9$ \\
\phantom{\agrill}$\tau=0.03$ & $8.9$ & $10.1$ & $10.3$ & $11.1$ & $10.1$ \\
\apatmat$\tau=0.01$ & $8.1$ & $7.1$ & $6.8$ & $5.6$ & $6.1$ \\
\phantom{\apatmat}$\tau=0.03$ & $6.3$ & $6.8$ & $5.5$ & $5.9$ & $5.2$ \\
\atopmeank$\tau=0.01$ & $9.1$ & $10.3$ & $8.8$ & $9.9$ & $10.7$ \\
\phantom{\atopmeank}$\tau=0.03$ & $9.4$ & $9.1$ & $8.3$ & $9.8$ & $10.2$ \\
\anpA$\tau=0.01$ & $9.6$ & $8.4$ & $9.4$ & $10.2$ & $9.8$ \\
\phantom{\anpA}$\tau=0.03$ & $10.8$ & $8.5$ & $9.3$ & $9.3$ & $7.6$ \\
\anpB$\tau=0.01$ & $6.6$ & \cellcolor{gray!30}$5.0$ & \cellcolor{gray!60}$3.4$ & \cellcolor{gray!60}$2.8$ & $4.4$ \\
\phantom{\anpB}$\tau=0.03$ & $6.9$ & $5.4$ & $4.7$ & $4.3$ & \cellcolor{gray!60}$3.2$ \\
\anpC$\tau=0.01$ & \cellcolor{gray!30}$4.9$ & $5.3$ & $6.3$ & $5.9$ & $6.0$ \\
\phantom{\anpC}$\tau=0.03$ & \cellcolor{gray!30}$4.6$ & \cellcolor{gray!60}$4.1$ & \cellcolor{gray!30}$4.0$ & \cellcolor{gray!30}$2.9$ & \cellcolor{gray!30}$4.1$ \\
\bottomrule
\end{tabular}
\end{table}

From these tables, we make several observations:
\begin{itemize}\itemsep 0pt
\item \toppushk performs better than \toppush. We relate this to the greater stability given by considering $k$ largest negatives in \toppushk instead of $1$ in \toppush. This may have alleviated the problems with $\bm w=0$ as suggested in Theorem \ref{thm:large_t}.
\item Neither \grill nor \npA perform well. We believe that this is due to the lack of convexity as indicated in Theorem \ref{thm:convex} and the discussion thereafter.
\item \topmeank does not perform well either. Since the thresholds $\tau$ are small, then $\bm w=0$ is the global minimum as proved in Corollary \ref{cor:topmean}.
\item \toppush and \toppushk methods work better on the Positives@Top criterion. This complies with the theory as these methods were designed for this criterion.
\item \patmat, \npB and \npC methods work better on the Positives@Quantile and Positives@NP criterion. This complies with the theory as these methods were designed for one of these criteria.
\end{itemize}

In Table \ref{tab:fails} we investigate the impact of $\bm w=0$ as a potential global minimum. Each method was optimized for six different values of hyperparameters. The table depicts the condition under which the final value has a lower objective that $\bm w=0$. Thus, \cmark\ means that it is always better while \xmark\ means that the algorithm made no progress from the starting point $\bm w =0$. The latter case implies that $\bm w=0$ seems to be the global minimum. We make the following observations:
\begin{itemize}\itemsep 0pt
\item \toppushk has a lower number of successes than \npB which corresponds to Figure \ref{fig:thresholds} showing that the latter method has a lower threshold.
\item Similarly, Figure \ref{fig:thresholds} states that the methods from Section \ref{sec:obj2} has a higher threshold than their Neyman-Pearson variants from Section \ref{sec:obj3}. This is documented in the table as the latter have a higher number of successes.
\item \patmat and \npB are the only methods which succeeded at every dataset for some hyperparameter. Moreover, for each dataset, there was some $\beta_0$ such that these methods were successful if and only if $\beta\in(0,\beta_0)$. This is in agreement with Theorem \ref{thm:patmat_zero}. The only exception was Spambase which we attribute to numerical errors.
\item \topmeank fails everywhere which agrees with Corollary \ref{cor:topmean}. The only exception is the case of $\tau=3\%$ and CTA which has only $2.1\%$ positive samples. This again agrees with Corollary \ref{cor:topmean}.
\end{itemize}

\begin{table}[!ht]
\caption{Necessary hyperparameter choice for the solution to have a better objective than zero. \cmark\ means that the solution was better than zero for all hyperparameters while \xmark\ means that it was worse for all hyperparameters.}
\label{tab:fails}
\centering
\begin{tabular}{@{}lllllll@{}}
\toprule
 & Ionosphere & Spambase & Gisette & Hepmass & CTA & NetFlow \\
\midrule
\atoppush & \cmark & \xmark & \cmark & \xmark & \cmark & \xmark \\
\atoppushk & \cmark & \xmark & \cmark & \xmark & \cmark & \xmark \\
\agrill$\tau=0.01$ & \xmark & \xmark & \xmark & \xmark & \cmark & \cmark \\
\phantom{\agrill}$\tau=0.03$ & \cmark & \xmark & \xmark & \xmark & \cmark & \cmark \\
\apatmat$\tau=0.01$ & $\beta\le0.1$ & $\beta=0.001$ & $\beta\le0.001$ & $\beta\le0.1$ & $\beta\le0.1$ & $\beta\le0.1$ \\
\phantom{\apatmat}$\tau=0.03$ & $\beta\le0.1$ & $\beta=0.01$ & $\beta\le0.001$ & $\beta\le0.1$ & $\beta\le0.1$ & $\beta\le0.1$ \\
\atopmeank$\tau=0.01$ & \xmark & \xmark & \xmark & \xmark & \xmark & \xmark \\
\phantom{\atopmeank}$\tau=0.03$ & \xmark & \xmark & \xmark & \xmark & \cmark & \xmark \\
\anpA$\tau=0.01$ & \cmark & \xmark & \xmark & \xmark & \cmark & \cmark \\
\phantom{\anpA}$\tau=0.03$ & \cmark & \xmark & \xmark & \xmark & \cmark & \cmark \\
\anpB$\tau=0.01$ & $\beta\le1$ & $\beta=0.01$ & \cmark & $\beta\le0.1$ & \cmark & $\beta\le0.1$ \\
\phantom{\anpB}$\tau=0.03$ & $\beta\le1$ & $\beta=0.1$ & \cmark & $\beta\le0.1$ & \cmark & $\beta\le0.1$ \\
\anpC$\tau=0.01$ & \cmark & $\lambda=0.001$ & \cmark & \xmark & \cmark & \xmark \\
\phantom{\anpC}$\tau=0.03$ & \cmark & \cmark & \cmark & \xmark & \cmark & \xmark \\
\bottomrule
\end{tabular}
\end{table}

The final Table \ref{tab:time} depicts the time needed for one iteration in milliseconds. We show four selected representative methods. Note that the time is relatively stable and even for most of the datasets it is below one millisecond. Thus, the whole optimization for one hyperparameter (without loading and evaluation) can be usually performed within one second.

\begin{table}[!ht]
\caption{Time in miliseconds needed for one iteration.}
\label{tab:time}
\centering
\begin{tabular}{@{}lllllll@{}}
\toprule
 & Ionosphere & Spambase & Gisette & Hepmass & CTA & NetFlow \\
\midrule
\atoppushk [ms] & $0.0$ & $0.1$ & $22.3$ & $5.6$ & $0.5$ & $0.9$ \\
\agrill  [ms] & $0.0$ & $0.1$ & $41.5$ & $6.2$ & $0.5$ & $1.1$ \\
\apatmat  [ms] & $0.0$ & $0.2$ & $38.4$ & $8.4$ & $0.4$ & $1.1$ \\
\atopmeank  [ms] & $0.0$ & $0.2$ & $43.9$ & $7.1$ & $0.4$ & $0.9$ \\
\bottomrule
\end{tabular}
\end{table}

\section{Conclusion}

In this paper we achieved the following results:
\begin{itemize}\itemsep 0pt
\item We presented a unified framework for the three criteria from Section \ref{sec:framework}.
\item We showed which known methods (\atoppush, \agrill) fall into our framework and derived both completely new methods (\apatmat, \anpB, \atopmeank) and modifications of known methods (\atoppushk, \anpA, \anpC).
\item We performed a theoretical analysis of the methods. We showed that both known methods suffer from certain disadvantages. While \toppush is sensitive to outliers, \grill is non-convex.
\item We performed a numerical comparison where we showed a good and fast performance. The methods converge within seconds on datasets with 5 million training samples.
\item The extensive theoretical analysis is supported by the numerical analysis.
\end{itemize}
Based on the results, we recommend using \toppushk or \npC for extremely small $\tau$. For larger $\tau$, we recommend using \patmat or \npB as they are convex and do not suffer from problems at $\bm w=0$.

% Acknowledgements should go at the end, before appendices and references

\paragraph{Acknowledgements} This work was supported by the Grant Agency of the Czech Republic (Grant No. 18-21409S), by the Program for Guangdong Introducing Innovative and Enterpreneurial Teams (Grant No. 2017ZT07X386) and by the Shenzhen Peacock Plan (Grant No. KQTD2016112514355531).

\appendix

\section{Additional results and proofs}\label{app:proofs}

Here, we provide additional results and proofs of results mentioned in the main body. For convenience, we repeat the result statements.

\subsection{Equivalence of \eqref{eq:problem_aatp_orig}, \eqref{eq:problem_aatp_grill} and \eqref{eq:problem_aatp}}

To show this equivalence, we will start with an auxiliary lemma.

\begin{lemma}\label{lemma:fnfp_equivalence}
Denote by $t$ the exact quantile from \eqref{eq:defin_quantile}. Then for all $\alpha\in[0,1]$ we have
\begin{equation}\label{eq:fnfp_equivalence}
\fp (\bm w,t) = \alpha \fp(\bm w,t) + (1-\alpha)\fn(\bm w,t) + (1-\alpha)(n\tau - n^+) + (1-\alpha)(q - 1),
\end{equation}
where $q:= \#\{\vec x\in\mathcal X|\ \bm w^\top \bm x= t\}$.
\end{lemma}
\begin{proof}
By the definition of the quantile we have
$$
\tp(\bm w,t)+\fp(\bm w,t) = n\tau + q-1.
$$
This implies
$$
\fp(\bm w,t)=n\tau +q - 1-\tp(\bm w,t) =n\tau+q-1-n^{+} +\fn(\bm w,t).
$$
From this relation we deduce
$$
\aligned
\fp(\bm w,t) &= \alpha \fp(\bm w,t) + (1-\alpha)\fp(\bm w,t) = \alpha \fp(\bm w,t) + (1-\alpha)\left(\fn(\bm w,t) + n\tau - n^+ + q - 1\right)\\
&= \alpha \fp(\bm w,t) + (1-\alpha)\fn(\bm w,t) + (1-\alpha)\left(n\tau - n^+\right) + (1-\alpha)\left(q - 1\right),
\endaligned
$$
which is precisely the lemma statement.
\end{proof}

The right-hand side of \eqref{eq:fnfp_equivalence} consists of three parts. The first one is a convex combination of false-positives and false-negatives and the second one is a constant term which has no impact on optimization. Finally, the third term $(1-\alpha)\left(q - 1\right)$ equals the number of samples for which their classifier equals the quantile. However, this term is small in comparison with the true-positives and the false-negatives and can be neglected. Moreover, when the data are ``truly'' random such as when measurement errors are present, then $q=1$ and this term vanishes completely. This gives the (almost) equivalence of \eqref{eq:problem_aatp_orig}, \eqref{eq:problem_aatp_grill} and \eqref{eq:problem_aatp}.

\subsection{Results related to convexity and continuity}

\theoremNN{thm:convex}{If the threshold $t$ is a convex function of the weights $\bm w$, then function $f(\bm w) = \fns(\bm w, t(\bm w))$ is convex.}
\begin{proof}
Due to the definition of the surrogate counts \eqref{eq:defin_counts_surr}, the objective of \eqref{eq:problem2} equals to
$$
\frac{1}{n^+}\sum_{\bm x\in\mathcal{X}^+} l\left(t(\bm w)-\bm w^\top \bm x\right).
$$
Here we write $t(\bm w)$ to stress the dependence of $t$ on $\bm w$. Since $\bm w\mapsto t(\bm w)$ is a convex function, we also have that $\bm w\mapsto t(\bm w)-\bm w^\top \bm x$ is a convex function. From its definition, the surrogate function $l$ is convex and non-decreasing. Since a composition of a convex function with a non-decreasing convex function is a convex function, this finishes the proof.
\end{proof}

\begin{lemma}\label{lemma:convex_surr}
Functions $\bm w\mapsto \bar t_{\rm Q}(\bm w)$ defined in \eqref{eq:defin_quantile_surr} and $\bm w\mapsto \bar t_{\rm NP}(\bm w)$ defined in \eqref{eq:defin_quantile_surr_np} are convex.
\end{lemma}
\begin{proof}
We will show this result only for the function $\bm w\mapsto \bar t_{\rm Q}(\bm w)$. For the second function, it can be shown in a identical way. This former function is defined via the implicit equation
$$
g(\bm w,t):=\frac{1}{n}\sum_{\bm x\in\mathcal{X}} l\left(\bm w^\top \bm x-t\right) - \tau = 0.
$$
Since $l$ is convex, we immediately obtain that $g$ is jointly convex in both variables.

To show the convexity, consider $\bm w_1$, $\bm w_2$ and the corresponding $t_1=\bar t_{\rm Q}(\bm w_1)$, $t_2=\bar t_{\rm Q}(\bm w_2)$. Note that this implies $g(\bm w_1,t_1)=g(\bm w_2,t_2)=0$. Then for any $\lambda\in[0,1]$ we have 
\begin{equation}\label{eq:proof_conv1}
g(\lambda\vec w_{1}+(1-\lambda)\vec w_{2},\;\lambda t_{1}+(1-\lambda)t_{2}) \le\lambda g(\vec w_{1},t_{1})+(1-\lambda)g(\vec w_{2},t_{2})=0,
\end{equation}
where the inequality follows from the convexity of $g$ and the equality
from $g(\vec w_{1},t_{1})=g(\vec w_{2},t_{2})=0.$
From the definition of the surrogate quantile we have
\begin{equation}\label{eq:proof_conv2}
g(\lambda\vec w_{1}+(1-\lambda)\vec w_{2},\bar{t}_{\rm Q}(\lambda\vec w_{1}+(1-\lambda)\vec w_{2}))=0.
\end{equation}
Since $g$ is non-increasing in the second variable, from \eqref{eq:proof_conv1} and \eqref{eq:proof_conv2} we deduce
$$
\bar{t}_{\rm Q}(\lambda\vec w_{1}+(1-\lambda)\vec w_{2})\le\lambda t_{1}+(1-\lambda)t_{2} =\lambda\bar{t}_{\rm Q}(\vec w_{1})+(1-\lambda)\bar{t}_{\rm Q}(\vec w_{2}),
$$
which implies that function $\vec w\mapsto \bar t_{\rm Q}(\bm w)$ is convex.
\end{proof}

\begin{lemma}\label{lemma:quantile_continuous}
Functions $\bm w\mapsto t_{\rm Q}(\bm w)$ defined in \eqref{eq:defin_quantile} and $\bm w\mapsto t_{\rm NP}(\bm w)$ defined in \eqref{eq:defin_quantile_np} are Lipschitz continuous.
\end{lemma}
\begin{proof}
The quantile $t_{\rm Q}(\bm w)$ equals to one or more scores $z=\bm w^\top \bm x$. This implies that function $\bm w\mapsto t_{\rm Q}(\bm w)$ is piecewise linear which implies that it is Lipschitz continuous.
\end{proof}

\subsection{Results related to the global minima at zero}

\theoremNN{thm:large_t}{
Consider any of these methods: \toppush, \toppushk, \topmeank or \npC. Fix any $\bm w$ and denote the corresponding threshold $t$. If we have
\begin{equation}\label{eq:w_zero_nn}
t\ge \frac{1}{n^+} \sum_{\bm x^+\in\Xcal^+} \bm w^\top \bm x^+,
\end{equation}
then $f(\bm 0)\le f(\bm w)$.
}
\begin{proof}
First note that due to $l(0)=1$ and the convexity of $l$ we have $l(z)\ge 1+cz$, where $c$ equals to the derivative of $l$ at $0$. Then we have
$$
\aligned
f(\bm w) &\ge \frac{1}{n^+}\fns(\bm w, t) = \frac{1}{n^+}\sum_{\bm x\in\Xcal^+}l(t-\bm w^\top \bm x) \ge \frac{1}{n^+}\sum_{\bm x\in\Xcal^+}(1+c(t-\bm w^\top \bm x)) \\
&= 1+\frac{c}{n^+}\sum_{\bm x\in\Xcal^+}(t-\bm w^\top \bm x) = 1+ct - \frac{c}{n^+}\sum_{\bm x\in\Xcal^+}\bm w^\top \bm x \ge 1,
\endaligned
$$
where the last inequality follows from \eqref{eq:w_zero_nn}. Now we realize that for any method from the statement, the corresponding threshold for $\bm w=0$ equals to $t=0$, and thus $f(\bm 0)=1$. But then $f(\bm 0)\le f(\bm w)$, which finishes the proof.
\end{proof}

\theoremNN{thm:patmat_zero}{
Consider the \patmat or \npB method with the hinge surrogate and no regularization. Assume that for some $\bm w$ we have
\begin{equation}\label{eq:patmat_zero_nn}
\frac{1}{n^+}\sum_{\bm x^+\in \Xcal^+}\bm w^\top \bm x^+ > \frac{1}{n^-}\sum_{\bm x^-\in \Xcal^-}\bm w^\top \bm x^-.
\end{equation}
Then there exists some $\beta>0$ such that $f(\bm w)<f(\bm 0)$.
}
\begin{proof}
Define
$$
\aligned
z_{\rm min} &= \min_{\bm x\in \Xcal}\bm w^\top \bm x, \\
\bar z &= \frac{1}{n}\sum_{\bm x\in \Xcal}\bm w^\top \bm x ,\\
z_{\rm max} &= \max_{\bm x\in \Xcal}\bm w^\top \bm x.
\endaligned
$$
Then we have the following chain of relations
\begin{equation}\label{eq:patmat_zero_aux0}
\aligned
\bar z &= \frac{1}{n}\sum_{\bm x\in\Xcal}\bm w^\top \bm x = \frac{1}{n}\sum_{\bm x\in\Xcal^+}\bm w^\top \bm x + \frac{1}{n}\sum_{\bm x\in\Xcal^-}\bm w^\top \bm x < \frac{1}{n}\sum_{\bm x\in\Xcal^+}\bm w^\top \bm x + \frac{n^-}{nn^+}\sum_{\bm x\in\Xcal^+}\bm w^\top \bm x \\
&= \frac 1n\left(1+\frac{n^-}{n^+}\right)\sum_{\bm x\in\Xcal^+}\bm w^\top \bm x = \frac 1n\frac{n^++n^-}{n^+}\sum_{\bm x\in\Xcal^+}\bm w^\top \bm x = \frac 1{n^+}\sum_{\bm x\in\Xcal^+}\bm w^\top \bm x.
\endaligned
\end{equation}
The only inequality follows from \eqref{eq:patmat_zero_nn} and the last equality follows from $n^++n^-=n$.

Due to \eqref{eq:patmat_zero_nn} we observe $z_{\rm min}<\bar z<z_{\rm max}$. Then we can define
$$
\aligned
\beta &= \min\left\{\frac{\tau}{\bar z-z_{\rm min}}, \frac{1-\tau}{z_{\rm max}-\bar z}\right\}, \\
t &= \frac{1-\tau}{\beta} + \bar z.
\endaligned
$$
We note that $\beta>0$. At the same time we obtain
\begin{equation}\label{eq:patmat_zero_aux1}
1+\beta(\bm w^\top\bm x-t) \ge 1+\beta(z_{\rm min}-t) = 1+\beta z_{\rm min}-1+\tau - \beta\bar z = \beta (z_{\rm min}-\bar z)+\tau\ge 0.
\end{equation}
Here, the first equality follows from the definition of $t$ and the last inequality from the definition of $\beta$. Then we have
$$
\aligned
\frac{1}{n}\sum_{\bm x\in\Xcal}l(\bm w^\top\bm x-t) &= \frac{1}{n}\sum_{\bm x\in\Xcal}\max\{1+\beta(\bm w^\top\bm x-t), 0\} = \frac{1}{n}\sum_{\bm x\in\Xcal}\left(1+\beta(\bm w^\top\bm x-t)\right) \\
&= 1-\beta t + \frac{\beta}{n}\sum_{\bm x\in\Xcal}\bm w^\top\bm x = 1-\beta t + \beta\bar z = \tau,
\endaligned
$$
where the second equality employs \eqref{eq:patmat_zero_aux1}, the third one the definition of $\bar z$ and the last one the definition of $t$. But this means that $t$ is the threshold corresponding to $\bm w$.

Similarly to \eqref{eq:patmat_zero_aux1} we get
$$
1+t-\bm w^\top \bm x \ge 1+t-z_{\rm max} = 1+\frac{1-\tau}{\beta} + \bar z - z_{\rm max} \ge \frac{1-\tau}{\beta} + \bar z - z_{\rm max} \ge0,
$$
where the last inequality follows from the definition of $\beta$. Then for the objective we have
\begin{equation}\label{eq:patmat_zero_aux2}
\aligned
f(\bm w) &= \frac{1}{n^+}\sum_{\bm x\in\Xcal^+}l(t-\bm w^\top \bm x) = \frac{1}{n^+}\sum_{\bm x\in\Xcal^+}\max\{1+t-\bm w^\top \bm x,0\} = \\
&= \frac{1}{n^+}\sum_{\bm x\in\Xcal^+}\left(1+t-\bm w^\top \bm x\right) =1+t-\frac{1}{n^+}\sum_{\bm x\in\Xcal^+}\bm w^\top \bm x < 1+t-\bar z \\
& = 1+\frac{1-\tau}{\beta} + \bar z -\bar z = 1+\frac{1-\tau}{\beta} = f(\bm 0),\\
\endaligned
\end{equation}
where we used \eqref{eq:patmat_zero_aux2} and \eqref{eq:patmat_zero_aux0} and the results for $\bm w=0$ from Section \ref{sec:example}. Thus, we finished the proof for \patmat. The proof for \npB can be performed in a identical way by replacing in the definition of $\bar z$ the mean with respect to all samples by the mean with respect to all negative samples.
\end{proof}

\subsection{Results related to threshold comparison}

\begin{lemma}\label{lemma:bound}
Denote the scores $z^+=\bm w^\top \bm x^+$ for $\bm x^+\in\Xcal^+$ and $z^-=\bm w^\top \bm x^-$ for $\bm x^-\in\Xcal^-$ and the ordered variants with decreasing components of $\bm z^-$ by $\bm z_{[\cdot]}^-$. Then we have the following implications
$$
\aligned
z_{[1]}^- \ge \frac{1}{n^+}\sum_{i=1}^{n^+} z_{i}^+ &\implies f(\bm 0)\le f(\bm w)\text{ for }\toppush, \\
\frac{1}{k}\sum_{i=1}^kz_{[i]}^- \ge \frac{1}{n^+}\sum_{i=1}^{n^+} z_{i}^+ &\implies f(\bm 0)\le f(\bm w)\text{ for }\toppushk, \\
\frac{1}{n^-\tau}\sum_{i=1}^{n^-\tau}z_{[i]}^- \ge \frac{1}{n^+}\sum_{i=1}^{n^+} z_{i}^+ &\implies f(\bm 0)\le f(\bm w)\text{ for }\npC. \\
\endaligned
$$
\end{lemma}
\begin{proof}
Since the left-hand side in the equality is the threshold for the corresponding methods, the result follows immediately from Theorem \ref{thm:large_t}.
\end{proof}

\begin{lemma}\label{lemma:thresholds1}
The threshold for the \patmat method is greater or equal than the threshold for the \topmeank method. Similarly, the threshold for the \npB method is greater or equal than the threshold for the \npC method. 
\end{lemma}
\begin{proof}
Define the scores $z=\bm w^\top \bm x$ and define $I$ to be the set of $n\tau$ indices where the scores have the largest values. Due to $l(0)=1$ and the convexity of $l$ we have $l(z)\ge 1+cz$, where $c$ is the derivative of $l$ at $0$. From the non-negativity of $l$ we have
$$
n\tau = \sum_{i=1}^nl(z_i - t) \ge \sum_{i\in I}l(z_i - t) \ge \sum_{i\in I}(1+c(z_i - t)) = n\tau - n\tau ct + c\sum_{i\in I}z_i,
$$
which implies
$$
t\ge \frac{1}{n\tau}\sum_{i\in I}z_i.
$$
But this finishes the proof for the \patmat method. The same result can be shown for the \npB method by considering only indices corresponding to negative samples.
\end{proof}

\begin{lemma}\label{lemma:thresholds2}
Define vector $\bm z^+$ with components $z^+=\bm w^\top \bm x^+$ for $\bm x^+\in\Xcal^+$ and similarly define vector $\bm z^-$ with components $z^-=\bm w^\top \bm x^-$ for $\bm x^-\in\Xcal^-$. Denote by $\bm z_{[\cdot]}^+$ and $\bm z_{[\cdot]}^-$ the sorted versions of $\bm z^+$ and $\bm z^-$, respectively. Then we have the following statements:
$$
\aligned
z_{[n^+\tau]}^+ > z_{[n^-\tau]}^- &\implies \grill\text{ has larger threshold than }\npA, \\
\frac{1}{n^+\tau}\sum_{i=1}^{n^+\tau} z_{[i]}^+ > \frac{1}{n^-\tau}\sum_{i=1}^{n^-\tau} z_{[i]}^-  &\implies \topmeank\text{ has larger threshold than }\npC. \\
\endaligned
$$
\end{lemma}
\begin{proof}
Since $\bm z^+$ and $\bm z^-$ are computed on disjunctive indices, we have
$$
z_{[n\tau]} \ge \min\{z_{[n^+\tau]}^+, z_{[n^-\tau]}^-\}.
$$
Since $z_{[n\tau]}$ is the threshold for the \grill method and $z_{[n^-\tau]}^-$ is the threshold for the \npA method, the first statement follows. The second part can be shown in a similar way.
\end{proof}

\noindent Since the goal of the presented methods is to push $z^+$ above $z^-$, we may expect that the conditions in Lemma \ref{lemma:thresholds2} hold true.

\section{Computation for Section \ref{sec:example}}\label{app:example}

Here, we derive the results presented in Section \ref{sec:example} more properly. We recall that we have $n$ negative samples randomly distributed in $[-1,0]\times[-1,1]$, $n$ positive samples randomly distributed in $[0,1]\times[-1,1]$ and one negative sample at $(2,0)$. We assume that $n$ is large and the outlier may be ignored for the computation of thresholds which require a large number of points.

For $\bm w_1=(0,0)$, the computation is rather clear. The threshold is obviously zero for all methods with the exception of \patmat, where we compute
$$
\tau = \frac{1}{n}\sum_{\bm x\in\Xcal}l(\beta(\bm w^\top \bm x-t)) = l(0-\beta t) = 1-\beta t,
$$
which implies $t=\frac{1}{\beta}(1-\tau)$. Moreover, for $t\ge 0$ we obtain 
$$
\frac{1}{n^+}\fn(\bm w_1,t) = \frac{1}{n^+}\sum_{\bm x\in \Xcal^+} l(t-0) = l(t) = 1+t.
$$
This finishes the computation for $\bm w_1$.

For $\bm w_2=(1,0)$ the computation goes similar. Then $\bm w^\top \bm x^+$ has the uniform distribution on $[0,1]$ while $\bm w^\top \bm x$ has the uniform distribution on $[-1,1]$. The computation of thresholds was explained in Section \ref{sec:example} with the exception of \patmat where for $\beta$ small enough we have
\begin{equation}\label{eq:example1}
\aligned
\tau &= \frac{1}{n}\sum_{\bm x\in\Xcal}l(\bm w_2^\top \bm x-t) \approx \int_{-1}^1 l(z-t)dz = \int_{-1}^1 \max\{0,1+\beta(z-t)\}dz \\
&=\int_{-1}^1 \left(1+\beta(z-t)\right)dz = 1-\beta t + \beta\int_{-1}^1zdz = 1-\beta t,
\endaligned
\end{equation}
and thus again $t=\frac{1}{\beta}(1-\tau)$. Note that
$$
1+\beta(z-t) \ge 1+\beta(-1-t) = 1-\beta - 1+\tau = -\beta+\tau
$$
and if $\beta\le\tau$, then we may indeed ignore the $\max$ operator in \eqref{eq:example1}. For the objective, for $t\ge 0$ we have
$$
\frac{1}{n^+}\fn(\bm w_2,t) \approx \int_0^1 l(t-z)dz = \int_0^1 \left(1+t-z\right)dz = 0.5+t.
$$

\section{Computation of Derivatives}\label{app:derivatives}

To apply the gradient descent method, we need to compute the derivatives. The general scheme was provided in \eqref{eq:derivatives}. It remains to compute only $\nabla t(\bm w)$, which we show here for \toppushk and \patmat. For other methods, it can be performed in a similar manner. 

For \toppushk, we recall that the threshold $t$ equals to the mean of $k$ largest scores $\bm w^\top \bm x$ for all negative samples $\bm x\in\Xcal^-$. The locally, this is a linear function for which it is simple to compute derivatives. To prevent unnecessary computations for \eqref{eq:derivatives}, we summarize the computations in Algorithm \ref{alg:derivatives1}. For simplicity, we denote by $I^+$ and $I^-$ the indices of positive and negative samples, respectively.

\begin{algorithm}[!ht]
\begin{algorithmic}[1]
\State Compute the scores $z_i = \bm w^\top \bm x_i$ for all indices $i\in I^+\cup I^-$
\State Find $k$ indices $l_1,\dots,l_k$ where $\bm z$ has the greatest values on $I^-$
\State Compute the threshold $t = \frac{1}{k}\sum_{j=1}^k z_{l_j}$
\State Compute the threshold derivative $\nabla t = \frac{1}{k}\sum_{j=1}^k \bm x_{l_j}$
%\State Compute the objective as $\frac{1}{n_+}\sum_{i\in I_+}l\left( t - z_i\right) + \frac\lambda2\norm{\bm w}^2$
\State Compute the derivative as $\frac{1}{n_+}\sum_{i\in I_+}l'\left( t- z_i\right) \left( \nabla t - \bm x_i\right)+\lambda\bm w$.
\end{algorithmic}
\caption{Efficient computation of \eqref{eq:derivatives} for \toppushk.}
\label{alg:derivatives1}
\end{algorithm}

For the \patmat method, the computation is slightly more difficult. The threshold $t$ is defined through equation
$$
\frac1n\sum_{\bm x\in\Xcal}l(\beta(\bm w^\top \bm x-t(\bm w)))=\tau.
$$
Differentiating this equation with respect to $\bm w$ results in
$$
\sum_{\bm x\in\Xcal}l'(\beta(\bm w^\top \bm x-t(\bm w)))(\bm x-\nabla t(\bm w))=0,
$$
and thus
$$
\nabla t(\bm w) = \frac{\sum_{\bm x\in\Xcal}l'(\beta(\bm w^\top \bm x-t(\bm w)))\bm x}{\sum_{\bm x\in\Xcal}l'(\beta(\bm w^\top \bm x-t(\bm w)))}.
$$
Similarly to the previous case, we present the efficient computation of \eqref{eq:derivatives} in Algorithm \ref{alg:derivatives2}.

\begin{algorithm}[!ht]
\begin{algorithmic}[1]
\State Compute the scores $z_i = \bm w^\top \bm x_i$ for all indices $i\in I^+\cup I^-$
\State Solve the equation $\sum_{i\in I^+\cup I^-}l(\beta(z_i-t))=n\tau$ for the threshold $t$
\State Compute the threshold derivative $\nabla t = \frac{\sum_{i\in I^+\cup I^-}l'(\beta(z_i-t))\bm x_i}{\sum_{i\in I^+\cup I^-}l'(\beta(z_i-t))}$
%\State Compute the objective as $\frac{1}{n_+}\sum_{i\in I_+}l\left( t - z_i\right)+ \frac\lambda2\norm{\bm w}^2$
\State Compute the derivative as $\frac{1}{n_+}\sum_{i\in I_+}l'\left( t- z_i\right) \left( \nabla t - \bm x_i\right)+\lambda \bm w$.
\end{algorithmic}
\caption{Efficient computation of \eqref{eq:derivatives} for \patmat.}
\label{alg:derivatives2}
\end{algorithm}

\section{Description of CTA and NetFlow datasets}\label{sec:cisco}

Suspicious network traffic is identified by a set of anomaly detection algorithms (detectors). Since these detectors differ from each, a common idea of ensemble learning is to (linearly) combine their results to obtain a better classifier. However, this is a non-trivial task as reported by \cite{Grill_2016}. The presented methods are suitable for this task as only a small fraction of the traffic is passed to a human for manual evaluation. 

We considered two real-world sources. The first one, further called \emph{CTA}, was created by the Cisco's Cognitive Threat Analytics engine~\cite{cta} which analyzes HTTP proxy logs (typically produced
by proxy servers located on a network perimeter). The second one, further called \emph{NetFlow}, was collected by the NetFlow anomaly detection engine~\cite{rehak2009adaptive,garcia2014empirical} which processes NetFlow~\cite{netflow} records exported by routers and other network traffic shaping devices. For both sources, the overall goal is to detect infected computers on the internal network and attacks on them. A more precise description of the used datasets and the data generating and labelling process is in~\cite{Grill_2016}.

All samples were manually labelled by experienced Cisco analysts.
Therefore it is safe to assume that samples labelled as malicious are
indeed malicious. However, it may happen that some samples labelled
as legitimate can be actually malicious, which typically happens 
due to the new type of attack (malware) with a different behaviour 
from what has been seen in the past. To study classifier robustness to this type of errors, we follow the experimental protocol of~\cite{Grill_2016} and besides the original dataset, we inject artificial noise \emph{MLT} which corresponds to the situation where security analysts have failed to identify 50\% of attack types and partially mislabelled the other 50\% attack types. Thus, some positive samples are either missing or have a wrong (negative) label in the training and validation set. However, they are present in the testing set.

\bibliographystyle{abbrv}
\bibliography{References}

\begin{thebibliography}{10}

\bibitem{agarwal2011infinite}
S.~Agarwal.
\newblock The infinite push: A new support vector ranking algorithm that
  directly optimizes accuracy at the absolute top of the list.
\newblock In {\em Proceedings of the 2011 SIAM International Conference on Data
  Mining}, pages 839--850. SIAM, 2011.

\bibitem{netflow}
E.~B.~Claise.
\newblock {Cisco Systems NetFlow services export version 9}.
\newblock 2004.

\bibitem{hepmass}
P.~Baldi, K.~Cranmer, T.~Faucett, P.~Sadowski, and D.~Whiteson.
\newblock Parameterized neural networks for high-energy physics.
\newblock {\em The European Physical Journal C}, 76(5):235, 2016.

\bibitem{bottou.2010}
L.~Bottou.
\newblock Large-scale machine learning with stochastic gradient descent.
\newblock In {\em Proceedings of COMPSTAT'2010}, pages 177--186. Springer,
  2010.

\bibitem{boyd2012accuracy}
S.~Boyd, C.~Cortes, M.~Mohri, and A.~Radovanovic.
\newblock Accuracy at the top.
\newblock In {\em Advances in neural information processing systems}, pages
  953--961, 2012.

\bibitem{boyd.2004}
S.~Boyd and L.~Vandenberghe.
\newblock {\em Convex optimization}.
\newblock Cambridge university press, 2004.

\bibitem{cta}
Cisco{~}Systems.
\newblock {CTA} {C}isco cognitive threat analytics on {C}isco cloud web
  security.
\newblock
  http://www.cisco.com/c/en/us/solutions/enterprise-networks/cognitive-threat-analytics,
  2014--2015.

\bibitem{davis2006relationship}
J.~Davis and M.~Goadrich.
\newblock The relationship between precision-recall and roc curves.
\newblock In {\em Proceedings of the 23rd international conference on Machine
  learning}, pages 233--240. ACM, 2006.

\bibitem{Eban_2017}
E.~Eban, M.~Schain, A.~Mackey, A.~Gordon, R.~Rifkin, and G.~Elidan.
\newblock Scalable learning of non-decomposable objectives.
\newblock In {\em Artificial Intelligence and Statistics}, pages 832--840,
  2017.

\bibitem{freund2003efficient}
Y.~Freund, R.~Iyer, R.~E. Schapire, and Y.~Singer.
\newblock An efficient boosting algorithm for combining preferences.
\newblock {\em The Journal of machine learning research}, 4:933--969, 2003.

\bibitem{garcia2014empirical}
S.~Garcia, M.~Grill, J.~Stiborek, and A.~Zunino.
\newblock An empirical comparison of botnet detection methods.
\newblock {\em Computers \& Security}, 45:100--123, 2014.

\bibitem{Grill_2016}
M.~Grill and T.~Pevn{\'y}.
\newblock Learning combination of anomaly detectors for security domain.
\newblock {\em Computer Networks}, 107:55--63, 2016.

\bibitem{gisette}
I.~Guyon, S.~Gunn, A.~Ben-Hur, and G.~Dror.
\newblock Result analysis of the nips 2003 feature selection challenge.
\newblock In {\em Advances in neural information processing systems}, pages
  545--552, 2005.

\bibitem{Joachims:2005:SVM:1102351.1102399}
T.~Joachims.
\newblock A support vector method for multivariate performance measures.
\newblock In {\em Proceedings of the 22nd International Conference on Machine
  Learning}, ICML '05, pages 377--384, New York, NY, USA, 2005. ACM.

\bibitem{kar2015surrogate}
P.~Kar, H.~Narasimhan, and P.~Jain.
\newblock Surrogate functions for maximizing precision at the top.
\newblock In {\em International Conference on Machine Learning}, pages
  189--198, 2015.

\bibitem{kingma2014adam}
D.~Kingma and J.~Ba.
\newblock Adam: A method for stochastic optimization.
\newblock {\em arXiv preprint arXiv:1412.6980}, 2014.

\bibitem{lapin.2015}
M.~Lapin, M.~Hein, and B.~Schiele.
\newblock Top-k multiclass svm.
\newblock In {\em Advances in Neural Information Processing Systems}, pages
  325--333, 2015.

\bibitem{lapin2018analysis}
M.~Lapin, M.~Hein, and B.~Schiele.
\newblock Analysis and optimization of loss functions for multiclass, top-k,
  and multilabel classification.
\newblock {\em IEEE transactions on pattern analysis and machine intelligence},
  40(7):1533--1554, 2018.

\bibitem{Li_TopPush}
N.~Li, R.~Jin, and Z.-H. Zhou.
\newblock Top rank optimization in linear time.
\newblock In {\em Proceedings of the 27th International Conference on Neural
  Information Processing Systems - Volume 1}, NIPS'14, pages 1502--1510,
  Cambridge, MA, USA, 2014. MIT Press.

\bibitem{mackey2018constrained}
A.~Mackey, X.~Luo, and E.~Eban.
\newblock Constrained classification and ranking via quantiles.
\newblock {\em arXiv preprint arXiv:1803.00067}, 2018.

\bibitem{rehak2009adaptive}
M.~Reh{\'a}k, M.~P{\v{e}}chou{\v{c}}ek, M.~Grill, J.~Stiborek, K.~Barto{\v{s}},
  and P.~{\v{C}}eleda.
\newblock Adaptive multiagent system for network traffic monitoring.
\newblock {\em IEEE Intelligent Systems}, (3):16--25, 2009.

\bibitem{rockafellar2000optimization}
R.~T. Rockafellar, S.~Uryasev, et~al.
\newblock Optimization of conditional value-at-risk.
\newblock {\em Journal of risk}, 2:21--42, 2000.

\bibitem{rudin.2009}
C.~Rudin.
\newblock The p-norm push: A simple convex ranking algorithm that concentrates
  at the top of the list.
\newblock {\em J. Mach. Learn. Res.}, 10:2233--2271, Dec. 2009.

\bibitem{ionosphere}
V.~G. Sigillito, S.~P. Wing, L.~V. Hutton, and K.~B. Baker.
\newblock Classification of radar returns from the ionosphere using neural
  networks.
\newblock {\em Johns Hopkins APL Technical Digest}, 10(3):262--266, 1989.

\bibitem{tasche2018plug}
D.~Tasche.
\newblock A plug-in approach to maximising precision at the top and recall at
  the top.
\newblock {\em arXiv preprint arXiv:1804.03077}, 2018.

\end{thebibliography}

\end{document}